\documentclass[onecolumn]{IEEEtran}
\usepackage{amsmath,amsfonts}
\usepackage{algorithmic}
\usepackage{algorithm}
\usepackage{array}
\usepackage[caption=false,font=normalsize,labelfont=sf,textfont=sf]{subfig}
\usepackage{textcomp}
\usepackage{stfloats}
\usepackage{url}
\usepackage{verbatim}
\usepackage{graphicx}
\usepackage{cite}
\usepackage{hyperref}
\usepackage{url}
\usepackage{booktabs}
\usepackage{bm}
\usepackage{multirow}
\usepackage{amsfonts}
\usepackage{float}
\usepackage{booktabs}
\usepackage{amssymb}
\usepackage{mathtools}
\usepackage{amsthm}
\usepackage{enumerate}
\usepackage{adjustbox}
\usepackage{xcolor}
\usepackage{varwidth}
\usepackage{capt-of}
\usepackage{pifont}
\usepackage{enumitem}
\usepackage{wrapfig}
\usepackage[]{mdframed}

\theoremstyle{plain}
\newtheorem{theorem}{Theorem}[section]
\newtheorem{restatetheoreminner}{Theorem}
\newenvironment{restatetheorem}[1]{%
  \renewcommand\therestatetheoreminner{#1}%
  \restatetheoreminner[Restate]
}{\endrestatetheoreminner}
\newtheorem{proposition}[theorem]{Proposition}
\newtheorem{restatepropositioninner}{Proposition}
\newenvironment{restateproposition}[1]{%
  \renewcommand\therestatepropositioninner{#1}%
  \restatepropositioninner[Restate]
}{\endrestatepropositioninner}
\newtheorem{lemma}[theorem]{Lemma}

\theoremstyle{definition}
\newtheorem{definition}[theorem]{Definition}
\newtheorem{assumption}[theorem]{Assumption}
\newtheorem{problem}[theorem]{Problem}
\theoremstyle{remark}

\newcommand{\cmark}{\ding{51}}%
\newcommand{\xmark}{\ding{55}}%

\def\1{\mathbf{1}}
\def\0{\mathbf{0}}
\def\E{\mathbb{E}}
\def\W{\mathbb{W}}
\def\tr{\mathrm{tr}}
\def\KL{\mathrm{KL}}
\def\SKL{\mathrm{SKL}}
\def\TV{\mathrm{TV}}
\def\Var{\mathrm{Var}}

\def\Pr{\mathbb{P}}
\def\tr{\mathrm{tr}}
\def\rank{\mathrm{rank}}

\DeclarePairedDelimiterX{\infdivx}[2]{(}{)}{%
  #1\,\delimsize\|\,#2%
}
\newcommand{\kl}{\KL\infdivx}
\newcommand{\skl}{\SKL\infdivx}
\newcommand*\dif{\mathop{}\!\mathrm{d}}

\DeclarePairedDelimiter{\abs}{\lvert}{\rvert}
\DeclarePairedDelimiter{\norm}{\lVert}{\rVert}

\DeclarePairedDelimiter{\prn}{\lparen}{\rparen}
\DeclarePairedDelimiter{\brk}{\lbrack}{\rbrack}
\DeclarePairedDelimiter{\brc}{\lbrace}{\rbrace}

\allowdisplaybreaks

\begin{document}

\title{How Does Distribution Matching Help Domain Generalization: An Information-theoretic Analysis}

\author{Yuxin Dong, Tieliang Gong, Hong Chen, Shuangyong Song, Weizhan Zhang,~\IEEEmembership{Senior Member,~IEEE,} Chen Li
\thanks{Corresponding author: Tieliang Gong.}%
\thanks{Y. Dong (yxdong9805@gmail.com), T. Gong (adidasgtl@gmail.com),  W. Zhang (zhangwzh@xjtu.edu.cn) and C. Li (cli@xjtu.edu.cn) are with the School of Computer Science and Technology, Xi'an Jiaotong University, Xi'an 710049, China.}
\thanks{H. Chen (chenh@mail.hzau.edu.cn) is with the College of Science, Huazhong Agriculture University, Wuhan 430070, China.}
\thanks{S. Song (songshy@chinatelecom.cn) is with the China Telecom Corporation, Beijing 100033, China.}
}

\markboth{IEEE Transactions}%
{Dong \MakeLowercase{\textit{et al.}}: How Does Distribution Matching Help Domain Generalization: An Information-theoretic Analysis}


\maketitle

\begin{abstract}
Domain generalization aims to learn invariance across multiple training domains, thereby enhancing generalization against out-of-distribution data. While gradient or representation matching algorithms have achieved remarkable success, these methods generally lack generalization guarantees or depend on strong assumptions, leaving a gap in understanding the underlying mechanism of distribution matching. In this work, we formulate domain generalization from a novel probabilistic perspective, ensuring robustness while avoiding overly conservative solutions. Through comprehensive information-theoretic analysis, we provide key insights into the roles of gradient and representation matching in promoting generalization. Our results reveal the complementary relationship between these two components, indicating that existing works focusing solely on either gradient or representation alignment are insufficient to solve the domain generalization problem. In light of these theoretical findings, we introduce IDM to simultaneously align the inter-domain gradients and representations. Integrated with the proposed PDM method for complex distribution matching, IDM achieves superior performance over various baseline methods.
\end{abstract}


\section{Introduction}
\IEEEPARstart{D}{istribution} shifts are prevalent in various real-world learning contexts, often leading to machine learning systems overfitting environment-specific correlations that may negatively impact performance when facing out-of-distribution (OOD) data \cite{geirhos2018imagenet, hendrycks2019benchmarking, azulay2019deep, hendrycks2021many}. Domain generalization (DG) is then introduced to address this challenge: By assuming the training data constitutes multiple domains that share some invariant underlying correlations, DG algorithms then attempt to learn this invariance so that domain-specific variations do not affect the model's performance. To this end, various DG approaches have been proposed, including invariant representation learning \cite{coral2016, mmd2018}, adversarial learning \cite{dann2016, cdann2018}, causal inference \cite{irm2019, causirl2022}, gradient manipulation \cite{iga2020, fish2021, fishr2022}, and robust optimization \cite{groupdro2019, vrex2021, qrm2022}.

DG is typically formulated as an average-case \cite{mtl2021, arm2021} or worst-case \cite{irm2019, groupdro2019} optimization problem, which however either lacks robustness against OOD data \cite{irm2019, nagarajan2020understanding} or leads to overly conservative solutions \cite{qrm2022}. In this paper, we introduce a novel probabilistic formulation that aims to minimize the gap between training and test-domain population risks with high probability. Our comprehensive generalization analysis then reveals that the input-output mutual information and the representation space covariate shift are pivotal in controlling this domain-level generalization gap, which could be achieved by aligning inter-domain gradients and representations, respectively.

Although distribution matching techniques are already well-explored in existing DG literature, these methods generally lack generalization guarantees or depend on strong assumptions, e.g. controllable invariant features \cite{iga2020}, quadratic bowl loss landscape \cite{fishr2022} or Lipschitz-continuous gradients \cite{andmask2020}. In contrast, we derive instructive generalization bounds by leveraging a relaxed i.i.d domain assumption \cite{qrm2022}, which is easily satisfied in practice. Our results indicate that combining gradient and representation matching effectively minimizes the domain-level generalization gap. Crucially, we reveal the complementary nature of these two components, highlighting that neither of them alone is sufficient to solve the DG problem.

In light of these theoretical findings, we propose inter-domain distribution matching (IDM) for high-probability DG by simultaneously aligning gradients and representations across training domains. Furthermore, we point out the limitations of traditional distribution alignment techniques, especially for high-dimensional and complex probability distributions. To circumvent these issues, we further propose per-sample distribution matching (PDM) by slicing and aligning individual sorted data points. IDM jointly working with PDM achieves superior performance on the Colored MNIST dataset \cite{irm2019} and the DomainBed benchmark \cite{domainbed2020}. Our primary contributions can be summarized as follows:
\begin{itemize}
    \item \textbf{Probabilistic formulation}: We introduce a probabilistic perspective for evaluating DG algorithms, focusing on their ability to minimize the domain-level generalization gap with high probability. Our approach leverages milder assumptions about the domains and enables generalization analysis with information-theoretic tools.
    \item \textbf{Information-theoretic insights}: Our analysis comprehensively elucidates the role of gradient and representation matching in promoting domain generalization. Most importantly, we reveal the complementary relationship between these two components, indicating that neither of them alone is sufficient to solve the DG problem.
    \item \textbf{Novel algorithms}: We propose IDM for high-probability DG by simultaneously aligning inter-domain gradients and representations, and PDM for complex distribution matching by slicing and aligning individual sorted data points. IDM jointly working with PDM achieves superior performance over various baseline methods.
\end{itemize}

\section{Problem Setting} \label{sec:problem}
We denote random variables by capitalized letters ($X$), their realizations by lower-case letters ($x$), and the corresponding spaces by calligraphic letters ($\mathcal{X}$). Let $\mathcal{Z} = \mathcal{X} \times \mathcal{Y}$ be the instance space of interest, where $\mathcal{X}$ and $\mathcal{Y}$ are the input space and the label space, respectively. Let $\mathcal{W}$ be the hypothesis space, each $w \in \mathcal{W}$ characterizes a predictor $f_w$: $\mathcal{X} \mapsto \mathcal{Y}$, comprised of an encoder $f_\phi$: $\mathcal{X} \mapsto \mathcal{R}$ and a classifier $f_\psi$: $\mathcal{R} \mapsto \mathcal{Y}$ with the assist of the representation space $\mathcal{R}$.

Following \cite{qrm2022}, we assume that there exists a distribution $\nu$ over the space of possible environments $\mathcal{D}$, where each domain $d \in \mathcal{D}$ corresponds to a specific data-generating distribution $\mu_d = P_{Z|D=d}$. The source $D_s = \{D_i\}_{i=1}^m$ and target $D_t = \{D_k\}_{k=1}^{m'}$ domains are both random variables sampled from $\nu$. Let $S = \{S_i\}_{i=1}^m$ denote the training dataset, with each subset $S_i = \{Z_j^i\}_{j=1}^n$ containing $n$ i.i.d data sampled from $\mu_{D_i}$. The task is to design algorithm $\mathcal{A}: \mathcal{D}^m \mapsto \mathcal{W}$, taking $D_s$ as the input (with proxy $S$) and providing possibly randomized hypothesis $W = \mathcal{A}(D_s)$. Given the loss function $\ell: \mathcal{Y} \times \mathcal{Y} \mapsto \mathbb{R}^+$, the ability of some hypothesis $w \in \mathcal{W}$ to generalize in average is evaluated by the global population risk:
\begin{equation*}
    L(w) = \E_{D \sim \nu} [L_D(w)] = \E_{Z \sim \mu} [\ell(f_w(X),Y)],
\end{equation*}
where $L_d(w) = \E_{Z \sim \mu_d} [\ell(f_w(X),Y)]$ is the domain-level population risk. Since $\nu$ is unknown, only the source and target-domain population risks are tractable in practice:
\begin{equation*}
    L_s(w) = \frac{1}{m} \sum_{i=1}^m L_{D_i}(w), L_t(w) = \frac{1}{m'} \sum_{k=1}^{m'} L_{D_k}(w).
\end{equation*}
\textbf{Main Assumptions.} We list the assumptions considered in our theoretical analysis as follows:
\begin{assumption} \label{asmp:independent}
    $D_t$ is independent of $D_s$.
\end{assumption}
\begin{assumption} \label{asmp:loss_bounded}
    $\ell(\cdot, \cdot)$ is bounded in $[0, M]$.
\end{assumption}
\begin{assumption} \label{asmp:loss_subgauss}
    $\ell(f_w(X),Y)$ is $\sigma$-subgaussian w.r.t $Z \sim \mu$ for any $w \in \mathcal{W}$.
\end{assumption}
\begin{assumption} \label{asmp:loss_dist}
    $\ell(\cdot, \cdot)$ is symmetric and satisfies the triangle inequality, i.e. for any $y_1, y_2, y_3 \in \mathcal{Y}$, $\ell(y_1, y_2) = \ell(y_2, y_1)$ and $\ell(y_1, y_2) \le \ell(y_1, y_3) + \ell(y_3, y_2)$.
\end{assumption}
\begin{assumption} \label{asmp:loss_lip}
    $\ell(f_w(X),Y)$ is $\beta$-Lipschitz w.r.t a metric $c$ on $\mathcal{Z}$ for any $w \in \mathcal{W}$, i.e. for any $z_1, z_2 \in \mathcal{Z}$, $\abs{\ell(f_w(x_1),y_1) + \ell(f_w(x_2),y_2)} \le \beta c(z_1, z_2)$.
\end{assumption}
Subgaussianity (Assumption \ref{asmp:loss_subgauss}) is one of the most common assumptions for information-theoretic generalization analysis \cite{xu2017information, negrea2019information, neu2021information, wang2021generalization}. Notably, Assumption \ref{asmp:loss_bounded} is a strengthened version of Assumption \ref{asmp:loss_subgauss}, since any $[0,M]$-bounded random variable is always $M/2$-subgaussian. Lipschitzness (Assumption \ref{asmp:loss_lip}) is a crucial prerequisite for stability analysis and has also been utilized in deriving Wasserstein distance generalization bounds \cite{hardt2016train, bassily2020stability, lei2021stability, rodriguez2021tighter, yang2021simple, yang2021stability}. Assumption \ref{asmp:loss_dist} is fulfilled when distance functions, such as mean absolute error (MAE) and 0-1 loss, are used as loss functions. This assumption has also been examined in previous studies \cite{mansour2009domain, shen2018wasserstein, wang2022information}.

\textbf{High-Probability DG.} The classical empirical risk minimization (ERM) technique, which minimizes the average-case risk: $\min_w L(w)$, is found ineffective in achieving invariance across different environments \cite{irm2019, nagarajan2020understanding}. To overcome this limitation, recent works \cite{vrex2021, ibirm2021, fish2021, fishr2022, bayesirm2022, sparseirm2022} have cast DG as a worst-case optimization problem: $\min_w \max_d L_d(w)$. However, this approach is generally impractical without strong assumptions made in the literature \cite{christiansen2021causal, qrm2022}, e.g. linearity of the underlying causal mechanism \cite{irm2019, vrex2021, ibirm2021}, or strictly separable spurious and invariant features \cite{sparseirm2022}. On the contrary, we propose the following high-probability objective by leveraging the mild Assumption \ref{asmp:independent}:
\begin{problem}{(High-Probability DG)} \label{prb:dg_hp}
    \begin{equation*}
        \min_{\mathcal{A}} \E[L_s(W)],\ \textrm{ s.t. }\ \Pr \{\abs{L_t(W) - L_s(W)} \ge \epsilon\} \le \delta.
    \end{equation*}
\end{problem}
This formulation is directly motivated by the intuition that the optimal algorithm $\mathcal{A}$ should be chosen in consideration of minimizing the generalization gap between training-domain $L_s(W)$ and test-domain $L_t(W)$ population risks. The probability is taken over both the sampled domains ($D_s$ and $D_t$) and the learning algorithm ($W$). Notably, our Assumption \ref{asmp:independent} is significantly weaker than the previously adopted i.i.d domain assumption \cite{qrm2022} by allowing correlations between source (or target) domains, and should be trivially satisfied in practice.

\section{Generalization Analysis} \label{sec:theory}
The primary goal of DG is to tackle the distribution shift problem, raised by the variation in the data-generating distribution $\mu_d$ for different environments $d$. This inconsistency can be quantified by the mutual information $I(Z;D)$ between the data pair $Z$ and the environment identifier $D$, which can be further decomposed into:
\begin{equation}
    I(Z;D) \textrm{ (distribution shift)} = I(X;D) \textrm{ (covariate shift)} + I(Y;D|X) \textrm{ (concept shift)}. \label{eq:dist_shift}
\end{equation}
While $D$ is binary to distinguish training and test samples in \cite{federici2021information, li2021learning}, we extend this concept to any discrete or continuous space, provided that each $d \in \mathcal{D}$ corresponds to a distinct data distribution $\mu_d$. The right hand side (RHS) characterizes the changes in the marginal input distribution $P_X$ (covariate shift) as well as the predictive distribution $P_{Y|X}$ (concept shift). We first show that the achievable level of average-case risk $L(w)$ is constrained by the degree of concept shift as following:
\begin{proposition} \label{prop:concept_shift}
    For any predictor $Q_{Y|X}$, we have 
    \begin{equation*}
        \kl{P_{Y|X,D}}{Q_{Y|X}} \ge I(Y;D|X).
    \end{equation*}
\end{proposition}
When $\ell$ represents the cross-entropy loss, the population risk of predictor $Q$ on domain $d$ can be represented as the KL divergence $\kl{P_{Y|X,D=d}}{Q_{Y|X}}$, provided that $H(Y|X,D) = 0$ (i.e. the label can be entirely inferred from $X$ and $D$). This implies that any model fitting well in training domains will suffer from strictly positive risks in test domains once concept shift is induced. This implication verifies the trade-off between optimization and generalization as we characterized in Problem \ref{prb:dg_hp}, and highlights the inherent difficulty of solving the DG problem.

\subsection{Decomposing the Generalization Gap}
We further demonstrate that by connecting source and target-domain population risks via the average-case risk $L(W)$, one can decompose the constraint of Problem \ref{prb:dg_hp} into training and test-domain generalization gaps. To be specific, since the predictor $W$ is trained on the source domains $D_s$, it is reasonable to assume that $W$ achieves lower population risks on $D_s$ than on average, i.e. $L_s(W) \le L(W)$. Moreover, since the sampling process of test domains is independent of the hypothesis, the test-domain population risk $L_t(W)$ is an unbiased estimate of $L(W)$. Combining these two observations, it is natural to assume that $L_s(W) \le L(W) \approx L_t(W)$, implying that the average-case risk $L(W)$ acts as a natural bridge between the two. For any constant $\lambda \in (0, 1)$, one can prove that:
\begin{equation*}
    \Pr\{\abs{L_s(W) - L_t(W)} \ge \epsilon\} \le \Pr\{\abs{L_s(W) - L(W)} \ge \lambda\epsilon\} + \Pr\{\abs{L_t(W) - L(W)} \ge (1-\lambda)\epsilon\}.
\end{equation*}
While the first event heavily correlates with the hypothesis $W$, the second event is instead hypothesis-independent. This observation inspires us to explore both hypothesis-based and hypothesis-independent bounds to address source and target-domain generalization errors, respectively.

\subsection{Source-domain Generalization} \label{sec:source_gen}
We first provide a sufficient condition for source-domain generalization. Our results are motivated by recent advancements in generalization analysis within the information-theoretic framework \cite{bu2020tightening, harutyunyan2021information}. Specialized to our problem, we quantify the changes in the hypothesis once the training domains are observed through the input-output mutual information $I(W;D_i)$:
\begin{theorem} \label{thm:pp_domain}
    If Assumption \ref{asmp:loss_bounded} holds, then 
    \begin{equation*}
        \Pr\brc*{\abs*{L_s(W) - L(W)} \ge \epsilon} \le \frac{M}{m\epsilon\sqrt{2}} \sum_{i=1}^m \sqrt{I(W,D_i)} + \frac{1}{\epsilon}\E_{W,D}\abs{L_D(W) - L(W)},
    \end{equation*}
    where $D \sim \nu$ is independent of $W$.
\end{theorem}
Intuitively, extracting correlations between $X$ and $Y$ that are invariant across training domains enhances the generalization ability of machine learning models. The mutual information $I(W;D_i)$ approaches zero when the correlations that a model learns from a specific training domain $D_i$ are also present in other training environments. This does not imply that the model learns nothing from $D_s$: by further assuming the independence of these domains, the summation of $I(W,D_i)$ can be relaxed to $I(W;D_s)$, which measures the actual amount of information learned by the model. By minimizing each $I(W;D_i)$ and $L_s(W)$ simultaneously, learning algorithms are encouraged to discard domain-specific correlations while preserving invariant ones and thus achieve high generalization performance.


Interestingly, the second term at the RHS of Theorem \ref{thm:pp_domain} is highly relevant to the target-domain generalization gap, so we will postpone related analysis to Section \ref{sec:target_gen}.

Next, we demonstrate that the minimization of $I(W;D_i)$ can be achieved by matching the conditional distributions of inter-domain gradients. To see this, we assume that $W$ is optimized by some noisy and iterative learning algorithms, e.g. stochastic gradient descent (SGD). Then the rule of updating $W$ at step $t$ through ERM can be formulated as:
\begin{equation*}
    W_t = W_{t-1} - \eta_t \sum_{i=1}^m g(W_{t-1}, B_t^i), \quad \mathrm{where} \quad g(w, B_t^i) = \frac{1}{m\abs{B_t^i}} \sum_{z \in B_t^i} \nabla_w \ell(f_w(x),y),
\end{equation*}
providing $W_0$ as the initial guess. Here, $\eta_t$ is the learning rate, and $B_t^i$ is the batch of data points randomly drawn from training environment $D_i$ to compute the gradient. Suppose that algorithm $\mathcal{A}$ finishes in $T$ steps, we then have:
\begin{theorem} \label{thm:mi_bound}
    Let $G_t = -\eta_t \sum_{i=1}^m g(W_{t-1},B_t^i)$, then 
    \begin{equation*}
        I(W_T;D_i) \le \sum_{t=1}^T I(G_t;D_i|W_{t-1}).
    \end{equation*}
\end{theorem}
Although our analysis is derived from the update rule of SGD, the same conclusion applies to a variety of iterative and noisy learning algorithms, e.g. SGLD and AdaGrad. Theorem \ref{thm:mi_bound} suggests that minimizing $I(G_t;D_i|W_{t-1})$ in each update penalizes $I(W_T;D_i)$ and thus leads to training-domain generalization. Notably, this conditional mutual information $I(G_t;D_i|W_{t-1})$ can be rewritten as the KL divergence $\kl{P_{G_t|D_i,W_{t-1}}}{P_{G_t|W_{t-1}}}$, which directly motivates matching the distributions of inter-domain gradients:
\begin{mdframed}
    Gradient matching promotes training-domain generalization when $\E_{W,D}\abs{L_D(W) - L(W)}$ is minimized.
\end{mdframed}
Intuitively, gradient alignment enforces the model to learn common correlations shared across training domains, thus preventing overfitting to spurious features and promoting invariance \cite{fish2021, fishr2022}.

We further present an alternative approach by assuming Lipschitzness instead of Subgaussianity, which usually leads to tighter bounds beyond information-theoretic measures:
\begin{theorem} \label{thm:pp_domain_wass}
    If $\ell(f_w(X),Y)$ is $\beta'$-Lipschitz w.r.t $w$, then
    \begin{equation*}
        \abs*{\E_{W,D_s}[L_s(W)] - \E_W[L(W)]} \le \frac{\beta'}{m} \sum_{i=1}^m \E_{D_i} [\W(P_{W|D_i}, P_W)].
    \end{equation*}
\end{theorem}
Besides the elegant symmetry compared to KL divergence metrics, Wasserstein distance bounds are generally considered to be tighter improvements over information-theoretic bounds. To see this, we assume that the adopted metric $c$ is discrete, which leads to the following reductions:
\begin{equation}
    \E_{D_i} [\W(P_{W|D_i},P_W)] = \E_{D_i} [\TV(P_{W|D_i},P_W)] \le \E_{D_i} \sqrt{\frac{1}{2} \kl{P_{W|D_i}}{P_W}} \le \sqrt{\frac{1}{2} I(W;D_i)}, \label{eq:wass_reduce}
\end{equation}
where $\TV$ is the total variation. These reductions confirm that the RHS of Theorem \ref{thm:pp_domain} also upper bounds other alternative measures of domain differences i.e. total variation and Wasserstein distance. This observation encourages us to directly penalize the mutual information $I(W;D_i)$, which is not only more stable for optimization \cite{nguyen2021kl, wang2022information} but also enables simultaneous minimization of these alternative metrics.

\subsection{Target-domain Generalization} \label{sec:target_gen}
We then investigate sufficient conditions for target-domain generalization. Since the training process is independent of the test domains, the predictor could be considered as some constant hypothesis $w \in \mathcal{W}$. It is straightforward to verify that $\E_{D_t}[L_t(w)] = L(w)$ due to the identical domain distribution $\nu$. We then establish the following bound for the target-domain generalization gap:
\begin{theorem} \label{thm:pp_con}
    If Assumption \ref{asmp:loss_subgauss} holds, then $\forall w \in \mathcal{W}$,
    \begin{equation*}
        \Pr\brc*{\abs{L_t(w) - L(w)} \ge \epsilon} \le \frac{\sigma}{\epsilon}\sqrt{2 I(Z;D)}.
    \end{equation*}
\end{theorem}
The result above can be interpreted from two perspectives. Firstly, evaluating the predictor $w$ on randomly sampled test environments reflects its ability to generalize on average, since $L_t(w)$ is an unbiased estimate of $L(w)$. Secondly, knowledge about $L(w)$ can be used to predict the ability of $w$ to generalize on unseen domains, which complements Theorem \ref{thm:pp_domain} in solving Problem \ref{prb:dg_hp}.

In Theorem \ref{thm:pp_con}, the probability of generalization is mainly controlled by the extent of distribution shift $I(Z;D)$. Notably, $I(Z;D)$ is an intrinsic property of the data collection procedure, and thus cannot be penalized from the perspective of learning algorithms. Fortunately, the encoder $\phi$ can be considered as part of the data preprocessing procedure, enabling learning algorithms to minimize the representation space distribution shift. Under the same conditions as Theorem \ref{thm:pp_con}, we have that for any classifier $\psi$:
\begin{equation*}
    \Pr\brc*{\abs{L_t(\psi) - L(\psi)} \ge \epsilon} \le \frac{\sigma}{\epsilon}\sqrt{2 I(R,Y;D)}.
\end{equation*}
Let $P_{R,Y}$ be the joint distribution by pushing forward $P_Z$ via the encoder as $R = f_\phi(X)$, the representation space distribution shift can then be decomposed into:
\begin{equation*}
    I(R,Y;D) \textrm{ (distribution shift)} = I(R;D) \textrm{ (covariate shift)} + I(Y;D|R) \textrm{ (concept shift)}.
\end{equation*}
This motivates us to simultaneously minimize the representation space covariate shift and concept shift to achieve test-domain generalization. We further demonstrate that bounding the covariate shift $I(R;D)$ solely is sufficient for target-domain generalization with Assumption \ref{asmp:loss_dist}:
\begin{theorem} \label{thm:pp_best}
    If Assumptions \ref{asmp:loss_bounded} and \ref{asmp:loss_dist} hold, then $\forall \psi$,
    \begin{gather*}
        \Pr\brc*{L_t(\psi) - L(\psi) \ge \epsilon} \le \frac{M}{\epsilon\sqrt{2}} \sqrt{I(R;D)} + \frac{2}{\epsilon}L^*,
    \end{gather*}
    where $L^* = \min_{f^*: \mathcal{R} \mapsto \mathcal{Y}} [L(f^*)]$.
\end{theorem}
Similarly, we further refine these test-domain generalization bounds by incorporating the more stringent Assumption \ref{asmp:loss_lip} in Appendix \ref{sec:discuss}. Theorem \ref{thm:pp_best} indicates that test-domain generalization is mainly controlled by the amount of covariate shift. Notably, the optimal classifier $f^*$ is chosen from the entire space of functions mapping from $\mathcal{R}$ to $\mathcal{Y}$. In the noiseless case where there exists a ground-truth labeling function $h^*$ such that $Y = h^*(R)$, we have $L^* = 0$. Therefore, $L^*$ serves as an indicator for the level of label noise in the data-generating distribution $\mu$. Moreover, the representation space covariate shift $I(R;D)$ is equivalent to the KL divergence $\kl{P_{R|D}}{P_{R}}$, which directly motivates matching the distributions of inter-domain representations:
\begin{mdframed}
    Representation matching promotes test-domain generalization when the level of label noise $L^*$ is low.
\end{mdframed}

A byproduct of the proof of Theorem \ref{thm:pp_best} is an upper bound on the expected absolute target-domain generalization error:
\begin{equation*}
    \E_{W,D}\abs{L_D(W) - L(W)} \le \frac{M}{\sqrt{2}}\sqrt{I(R;D)} + 2L^*.
\end{equation*}
This result complements the training-domain generalization bound in Theorem \ref{thm:pp_domain}, confirming that penalizing $I(W;D_i)$ and $I(R;D)$ simultaneously is sufficient to minimize the training-domain generalization gap.

While minimizing $I(R;D)$ guarantees test-domain generalization, this operation requires knowledge about test-domain samples, which is not available during the entire training process. A walkaround is to match the training-domain representations instead, utilizing $I(R_i;D_i)$ for each $D_i \in D_s$ as a proxy to penalize $I(R;D)$. The following proposition verifies the feasibility of this approach:
\begin{proposition} \label{prop:cov_shift_emp}
    Under mild conditions, we have
    \begin{equation*}
        \skl{P_{R,D}}{P_{R_i,D_i}} = O\prn*{\sqrt{I(W;D_i)}},
    \end{equation*}
    where $\skl{P}{Q} = \kl{P}{Q} + \kl{Q}{P}$.
\end{proposition}
Interestingly, the discrepancy between the test-domain joint distribution of $P_{R,D}$ and its training-domain counterpart $P_{R_i,D_i}$ can be upper bounded by the input-output mutual information $I(W;D_i)$. By simultaneously minimizing $I(W;D_i)$, one can use $I(R_i;D_i)$ as a proxy to penalize $I(R;D)$ and achieve test-domain generalization.

While our analysis does not necessitate the independence condition between source domains or target domains, such a condition is also naturally satisfied in most learning scenarios and leads to tighter generalization bounds. Specifically, Theorem \ref{thm:pp_con} and \ref{thm:pp_best} can be further tightened by a factor of $\frac{1}{m'}$ when test domains are i.i.d. We refer the readers to Appendix \ref{sec:proof1} and \ref{sec:proof2} for the proof of these results.

\section{Inter-domain Distribution Matching} \label{sec:alg_design}
Motivated by our theoretical analysis in Section \ref{sec:theory}, we propose inter-domain distribution matching (IDM) to achieve high-probability DG (Problem \ref{prb:dg_hp}). Recall that the average-case risk $L(W)$ serves as a natural bridge to connect $L_s(W)$ and $L_t(W)$, the regularization in Problem \ref{prb:dg_hp} directly indicates an objective for optimization by combining the high-probability concentration bounds in Theorem \ref{thm:pp_domain} and \ref{thm:pp_con}. Specifically, for any $\lambda \in (0,1)$, we have:
\begin{equation} \label{eq:solution}
    \Pr\{\abs{L_t(W) - L_s(W)} \ge \epsilon\} \le \frac{M}{m\epsilon\lambda\sqrt{2}} \sum_{i=1}^m \sqrt{I(W,D_i)} + \frac{1}{\epsilon\lambda(1-\lambda)} \prn*{\frac{M}{\sqrt{2}}\sqrt{I(R;D)} + 2L^*}.
\end{equation}
This observation directly motivates aligning inter-domain distributions of the gradients and representations simultaneously. While the idea of distribution matching is not new, we are the first to explore the complementary relationship between gradient and representation matching:
\begin{mdframed}
    Gradient and representation matching together minimize $\Pr \{\abs{L_t(W) - L_s(W)} \ge \epsilon\}$ in Problem \ref{prb:dg_hp}.
\end{mdframed}
Specifically, training-domain generalization requires the minimization of the test-domain generalization gap (Theorem \ref{thm:pp_domain}), and test-domain generalization requires the minimization of the input-output mutual information (Proposition \ref{prop:cov_shift_emp}). Therefore, existing works focusing exclusively on either gradient or representation alignment are insufficient to fully address the domain-level generalization gap. To our best knowledge, this is the first exploration in the literature where gradient and representation matching are combined to yield a sufficient solution for the DG problem.

\subsection{Per-sample Distribution Matching}
While various distribution matching methods have been proposed in the literature, these techniques are generally either ineffective or insufficient for high-dimensional and complex distributions. Typically, learning algorithms have no knowledge about the underlying distribution of either the representation or the gradient, and the only available way is to align them across batched data points. We first provide an impossibility theorem for high-dimensional distribution matching in the cases of limited number of samples:
\begin{theorem}{(Informal)} \label{thm:indist}
    Let $n$ and $b$ be the dimension and the number of data points respectively. If $n > b + 1$, then there exist infinite environments whose conditional probabilities given an arbitrarily sampled group of data points are indistinguishable. If $n > 2b + 1$, then there exist infinite environments whose conditional probabilities cannot distinguish two arbitrarily sampled groups of data points.
\end{theorem}
We refer the readers to Appendix \ref{sec:proof2} for a formal statement of Theorem \ref{thm:indist}. In real-world scenarios, the dimensionality of the feature or the gradient easily exceeds that of the batch size, making algorithms that aim to align the entire distribution (e.g. CORAL \cite{coral2016} and MMD \cite{mmd2018}) generally ineffective since distribution alignment is basically impossible given such few data points. This observation is also verified by \cite{fishr2022} that aligning the entire covariance matrix achieves no better performance than aligning the diagonal elements only. Furthermore, prior distribution alignment techniques mainly focus on aligning the directions \cite{andmask2020, sandmask2021, fish2021} or low-order moments \cite{coral2016, iga2020, fishr2022}, which are insufficient for complex probability distributions. For example, while the standard Gaussian distribution $N(0,1)$ and the uniform distribution $U(-\sqrt{3}, \sqrt{3})$ share the same expectation and variance, they are fundamentally different to one another. To address these issues, we propose the per-sample distribution matching (PDM) technique that aligns distributions in a per-dimension manner, by minimizing an upper bound of the KL divergence between probability density estimators.

Let $\{x_i^1\}_{i=1}^b$ and $\{x_i^2\}_{i=1}^b$ be two groups of 1-dimensional data points drawn from probability distributions $P$ and $Q$ respectively. Let $p_i$ denote the density of Gaussian distribution with expectation $x_i^1$ and variance $\sigma^2$, then the kernel density estimator $\bar{P}$ for $P$ can be written as $\bar{p}(x) = \frac{1}{b} \sum_i p_i(x)$ (respectively for $q_i$, $\bar{Q}$ and $\bar{q}$). The following theorem suggests a computable upper bound for the KL divergence (Wasserstein distance) between probability density estimators:
\begin{theorem} \label{thm:bijection}
    Let $f$ be a bijection: $[1,b] \leftrightarrow [1,b]$ and $P_i$ be the probability measure defined by $p_i$ (respectively for $Q_i$ and $q_i$), then $\kl{\bar{P}}{\bar{Q}} \le \frac{1}{b} \sum_{i=1}^b \kl{P_i}{Q_{f(i)}}$, and $\W(\bar{P}, \bar{Q}) \le \frac{1}{b} \sum_{i=1}^b \W(P_i, Q_{f(i)})$.
\end{theorem}
Hence, distribution matching can be achieved by minimizing the KL divergence or Wasserstein distances between point Gaussian densities, which can be achieved by aligning individual data points. The following theorem suggests an optimal bijection for choosing the order of alignment:
\begin{theorem} \label{thm:minimizer}
    Suppose that $\{x_i^1\}_{i=1}^b$ and $\{x_i^2\}_{i=1}^b$ are both sorted in the same order, then $f(j) = j$ is the minimizer of both $\sum_{i=1}^b \kl{P_i}{Q_{f(i)}}$ and $\sum_{i=1}^b \W(P_i, Q_{f(i)})$.
\end{theorem}
To summarize, the procedure of PDM is to slice the data points into separate dimensions, sort the data points in ascending (or descending) order for each dimension, and then match the sorted data points across different training domains. PDM improves over previous distribution matching techniques by simultaneously capturing multiple orders of moments, avoiding ineffective high-dimensional distribution matching, as well as enabling straightforward implementation and efficient computation. We provide pseudo-codes for both PDM and IDM in Appendix \ref{sec:setting}.


\subsection{Algorithm Design}
Combining the methods discussed above, we finally propose the IDM algorithm for high-probability DG by simultaneously aligning inter-domain gradients and representations. Recall that Problem \ref{prb:dg_hp} incorporates an additional regularization based on ERM, we adopt the following Lagrange multipliers to optimize the IDM objective:
\begin{equation}
    \mathcal{L}_{\mathrm{IDM}} = \mathcal{L}_{\mathrm{E}} + \lambda_1 \mathcal{L}_{\mathrm{G}} + \lambda_2 \mathcal{L}_{\mathrm{R}} = \frac{1}{m} \sum_{i=1}^m \brk*{L_{D_i}(W) + \lambda_1 \mathrm{PDM}(G_i) + \lambda_2 \mathrm{PDM}(R_i)}. \label{eq:penalty}
\end{equation}
Here $\mathcal{L}_{\mathrm{E}}$ is the risk of ERM, $\mathcal{L}_{\mathrm{G}}$ and $\mathcal{L}_{\mathrm{R}}$ denote the penalty of distribution matching for the gradients and representations respectively, implemented with the proposed PDM method. To cooperate representation alignment which regards the classifier $\psi$ as the true predictor and also for memory and time concerns, we only apply gradient alignment for the classifier $\psi$ as in \cite{fishr2022}. Furthermore, $\lambda_1$ and $\lambda_2$ should be adaptively chosen according to the extent of covariate and concept shifts respectively: Firstly, $I(R;D)$ is upper bounded by $I(X;D)$ by the Markov chain $D \rightarrow X \rightarrow R$, so the representations are naturally aligned when $I(X;D) = 0$. Secondly, gradient alignment is not required when $I(Y;D|X) = 0$, since the distribution shift can then be minimized by aligning the representations solely. Therefore, $\lambda_1$ and $\lambda_2$ should scale with the amount of the covariate and concept shifts, respectively.

\begin{figure*}[t]
\begin{minipage}[t]{0.6\linewidth}
\vspace{0pt}
\centering
\includegraphics[width=\linewidth]{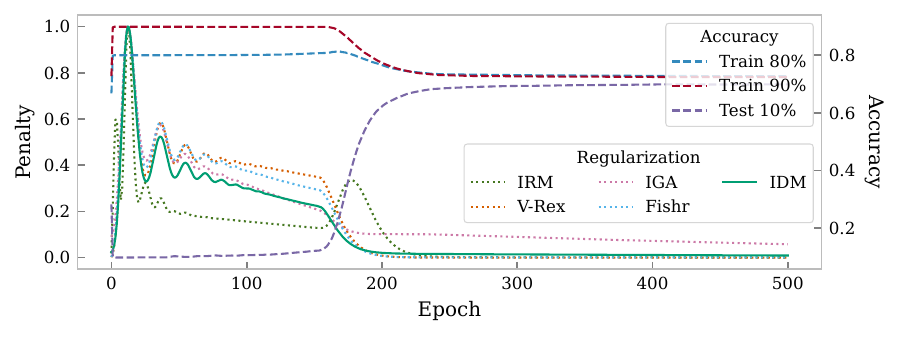}
\vspace{-20pt}
\captionof{figure}{Learning dynamics of IDM.}
\label{fig:cmnist}
\end{minipage}%
\begin{minipage}[t]{0.4\linewidth}
\vspace{0pt}\raggedright
\centering
\small
\captionof{table}{The Colored MNIST task.}
\label{tbl:cmnist}
\adjustbox{max width=\textwidth}{%
\begin{tabular}{lccc}
\toprule
Method & Train Acc & Test Acc & Gray Acc \\
\midrule
ERM & 86.4 \scriptsize{$\pm$ 0.2} & 14.0 \scriptsize{$\pm$ 0.7} & 71.0 \scriptsize{$\pm$ 0.7} \\
IRM & 71.0 \scriptsize{$\pm$ 0.5} & 65.6 \scriptsize{$\pm$ 1.8} & 66.1 \scriptsize{$\pm$ 0.2} \\
V-REx & 71.7 \scriptsize{$\pm$ 1.5} & 67.2 \scriptsize{$\pm$ 1.5} & 68.6 \scriptsize{$\pm$ 2.2} \\
IGA & 68.9 \scriptsize{$\pm$ 3.0} & 67.7 \scriptsize{$\pm$ 2.9} & 67.5 \scriptsize{$\pm$ 2.7} \\
Fishr & 69.6 \scriptsize{$\pm$ 0.9} & 71.2 \scriptsize{$\pm$ 1.1} & 70.2 \scriptsize{$\pm$ 0.7} \\
\midrule
IDM & 70.2 \scriptsize{$\pm$ 1.4} & 70.6 \scriptsize{$\pm$ 0.9} & 70.5 \scriptsize{$\pm$ 0.7} \\
\bottomrule
\end{tabular}}
\end{minipage}
\end{figure*}

\section{Related Works}
In the literature, various approaches have been proposed by incorporating external domain information to achieve OOD generalization. Most recent works achieve invariance by employing additional regularization criteria based on ERM. These methods differ in the choice of the statistics used to match across training domains and can be categorized by the corresponding objective of 1) gradient, 2) representation, and 3) predictor, as follows:

\textbf{Invariant Gradients.} Gradient alignment enforces batched data points from different domains to cooperate and promotes OOD generalization by finding loss minima shared across training domains. Specifically, IGA \cite{iga2020} aligns the empirical expectations, Fish \cite{fish2021} maximizes the dot-product of inter-domain gradients, AND-mask \cite{andmask2020} and SAND-mask \cite{sandmask2021} only update weights when the gradients share the same direction, and Fishr \cite{fishr2022} matches the gradient variance. These gradient-based objectives are generally restricted to aligning the directions or low-order moments, resulting in substantial information loss in more granular statistics. Besides, these works either lack generalization guarantees or rely on strong assumptions including the existence of invariant and controllable features (IGA), the shape of loss landscape around the local minima (Fishr), Lipschitz continuous gradients and co-diagonalizable Hessian matrix (AND-mask). In contrast, our analysis successfully connects gradient alignment and training-domain generalization by leveraging the mild assumption of identical domain distributions.

\textbf{Invariant Representations.} Extracting domain-invariant features has been extensively studied to solve both DG and domain adaptation (DA) problems. DANN \cite{dann2016} and CDANN \cite{cdann2018} align inter-domain representations via adversarial learning, MMD \cite{mmd2018} uses kernel methods for distribution alignment, and CORAL \cite{coral2016} matches low-order moments of the representations. Still, these methods are insufficient for complex probability distributions \cite{zhao2019learning}, ineffective for high-dimensional distributions (Theorem \ref{thm:indist}), and incapable of addressing the concept shift. Besides, the viability of minimizing the representation shift $I(R;D)$ through the training-domain proxy $I(R_i;D_i)$ remains questionable without gradient matching (Proposition \ref{prop:concept_shift}). Our analysis sheds light on understanding how representation alignment enhances test-domain generalization by minimizing the variance of target-domain risks.

\textbf{Invariant Predictors.} A recent line of works proposes to explore the connection between invariance and causality. IRM \cite{irm2019} and subsequent works \cite{sparseirm2022,bayesirm2022} learn an invariant classifier that is simultaneously optimal for all training domains. However, later works have shown that IRM may fail on non-linear data and lead to sub-optimal predictors \cite{kamath2021does, ibirm2021}. Parallel works include: V-REx \cite{vrex2021} which minimizes the variance of training-domain risks, GroupDRO \cite{groupdro2019} which minimizes the worst-domain training risk, and QRM \cite{qrm2022} which optimizes a quantile of the risk distribution. As shown in the next section, IDM also promotes domain-invariant predictors and ensures optimality across different training domains.

\begin{table*}[t]
\centering
\small
\caption{The DomainBed benchmark. We format \textbf{best}, \underline{second best} and \textcolor{gray}{worse than ERM} results.}
\label{tbl:domainbed_oracle}
\adjustbox{max width=\textwidth}{%
\begin{tabular}{l|cccccccc|ccc}
\toprule
\multirow{2}{*}{Algorithm} & \multicolumn{8}{c|}{Accuracy ($\uparrow$)} & \multicolumn{3}{c}{Ranking ($\downarrow$)} \\
& \textbf{CMNIST} & \textbf{RMNIST} & \textbf{VLCS} & \textbf{PACS} & \textbf{OffHome} & \textbf{TerraInc} & \textbf{DomNet} & \textbf{Avg} & \textbf{Mean} & \textbf{Median} & \textbf{Worst} \\
\midrule
ERM & 57.8 \scriptsize{$\pm$ 0.2} & 97.8 \scriptsize{$\pm$ 0.1} & 77.6 \scriptsize{$\pm$ 0.3} & 86.7 \scriptsize{$\pm$ 0.3} & 66.4 \scriptsize{$\pm$ 0.5} & 53.0 \scriptsize{$\pm$ 0.3} & 41.3 \scriptsize{$\pm$ 0.1} & 68.7 & 12.3 & 11 & 20 \\
IRM & 67.7 \scriptsize{$\pm$ 1.2} & \textcolor{gray}{97.5} \scriptsize{$\pm$ 0.2} & \textcolor{gray}{76.9} \scriptsize{$\pm$ 0.6} & \textcolor{gray}{84.5} \scriptsize{$\pm$ 1.1} & \textcolor{gray}{63.0} \scriptsize{$\pm$ 2.7} & \textcolor{gray}{50.5} \scriptsize{$\pm$ 0.7} & \textcolor{gray}{28.0} \scriptsize{$\pm$ 5.1} & \textcolor{gray}{66.9} & \textcolor{gray}{18.3} & \textcolor{gray}{20} & \textcolor{gray}{22} \\
GroupDRO & 61.1 \scriptsize{$\pm$ 0.9} & 97.9 \scriptsize{$\pm$ 0.1} & \textcolor{gray}{77.4} \scriptsize{$\pm$ 0.5} & 87.1 \scriptsize{$\pm$ 0.1} & \textcolor{gray}{66.2} \scriptsize{$\pm$ 0.6} & \textcolor{gray}{52.4} \scriptsize{$\pm$ 0.1} & \textcolor{gray}{33.4} \scriptsize{$\pm$ 0.3} & \textcolor{gray}{67.9} & 11.7 & 10 & 19 \\
Mixup & 58.4 \scriptsize{$\pm$ 0.2} & 98.0 \scriptsize{$\pm$ 0.1} & 78.1 \scriptsize{$\pm$ 0.3} & 86.8 \scriptsize{$\pm$ 0.3} & 68.0 \scriptsize{$\pm$ 0.2} & \textbf{54.4} \scriptsize{$\pm$ 0.3} & \textcolor{gray}{39.6} \scriptsize{$\pm$ 0.1} & 69.0 & 7.3 & 6 & 15 \\
MLDG & 58.2 \scriptsize{$\pm$ 0.4} & 97.8 \scriptsize{$\pm$ 0.1} & \textcolor{gray}{77.5} \scriptsize{$\pm$ 0.1} & 86.8 \scriptsize{$\pm$ 0.4} & 66.6 \scriptsize{$\pm$ 0.3} & \textcolor{gray}{52.0} \scriptsize{$\pm$ 0.1} & 41.6 \scriptsize{$\pm$ 0.1} & 68.7 & \textcolor{gray}{12.6} & \textcolor{gray}{13} & 18 \\
CORAL & 58.6 \scriptsize{$\pm$ 0.5} & 98.0 \scriptsize{$\pm$ 0.0} & 77.7 \scriptsize{$\pm$ 0.2} & 87.1 \scriptsize{$\pm$ 0.5} & \textbf{68.4} \scriptsize{$\pm$ 0.2} & \textcolor{gray}{52.8} \scriptsize{$\pm$ 0.2} & 41.8 \scriptsize{$\pm$ 0.1} & 69.2 & 6.4 & 5 & \underline{14} \\
MMD & 63.3 \scriptsize{$\pm$ 1.3} & 98.0 \scriptsize{$\pm$ 0.1} & 77.9 \scriptsize{$\pm$ 0.1} & 87.2 \scriptsize{$\pm$ 0.1} & \textcolor{gray}{66.2} \scriptsize{$\pm$ 0.3} & \textcolor{gray}{52.0} \scriptsize{$\pm$ 0.4} & \textcolor{gray}{23.5} \scriptsize{$\pm$ 9.4} & \textcolor{gray}{66.9} & 10.0 & 10 & \textcolor{gray}{22} \\
DANN & \textcolor{gray}{57.0} \scriptsize{$\pm$ 1.0} & 97.9 \scriptsize{$\pm$ 0.1} & \underline{79.7} \scriptsize{$\pm$ 0.5} & \textcolor{gray}{85.2} \scriptsize{$\pm$ 0.2} & \textcolor{gray}{65.3} \scriptsize{$\pm$ 0.8} & \textcolor{gray}{50.6} \scriptsize{$\pm$ 0.4} & \textcolor{gray}{38.3} \scriptsize{$\pm$ 0.1} & \textcolor{gray}{67.7} & \textcolor{gray}{15.0} & \textcolor{gray}{18} & \textcolor{gray}{22} \\
CDANN & 59.5 \scriptsize{$\pm$ 2.0} & 97.9 \scriptsize{$\pm$ 0.0} & \textbf{79.9} \scriptsize{$\pm$ 0.2} & \textcolor{gray}{85.8} \scriptsize{$\pm$ 0.8} & \textcolor{gray}{65.3} \scriptsize{$\pm$ 0.5} & \textcolor{gray}{50.8} \scriptsize{$\pm$ 0.6} & \textcolor{gray}{38.5} \scriptsize{$\pm$ 0.2} & \textcolor{gray}{68.2} & \textcolor{gray}{12.4} & \textcolor{gray}{14} & 18 \\
MTL & \textcolor{gray}{57.6} \scriptsize{$\pm$ 0.3} & 97.9 \scriptsize{$\pm$ 0.1} & 77.7 \scriptsize{$\pm$ 0.5} & 86.7 \scriptsize{$\pm$ 0.2} & 66.5 \scriptsize{$\pm$ 0.4} & \textcolor{gray}{52.2} \scriptsize{$\pm$ 0.4} & \textcolor{gray}{40.8} \scriptsize{$\pm$ 0.1} & \textcolor{gray}{68.5} & 11.7 & 10 & \textcolor{gray}{21} \\
SagNet & 58.2 \scriptsize{$\pm$ 0.3} & 97.9 \scriptsize{$\pm$ 0.0} & 77.6 \scriptsize{$\pm$ 0.1} & \textcolor{gray}{86.4} \scriptsize{$\pm$ 0.4} & 67.5 \scriptsize{$\pm$ 0.2} & \textcolor{gray}{52.5} \scriptsize{$\pm$ 0.4} & \textcolor{gray}{40.8} \scriptsize{$\pm$ 0.2} & 68.7 & 11.3 & 9 & 17 \\
ARM & 63.2 \scriptsize{$\pm$ 0.7} & \textbf{98.1} \scriptsize{$\pm$ 0.1} & 77.8 \scriptsize{$\pm$ 0.3} & \textcolor{gray}{85.8} \scriptsize{$\pm$ 0.2} & \textcolor{gray}{64.8} \scriptsize{$\pm$ 0.4} & \textcolor{gray}{51.2} \scriptsize{$\pm$ 0.5} & \textcolor{gray}{36.0} \scriptsize{$\pm$ 0.2} & \textcolor{gray}{68.1} & \textcolor{gray}{13.0} & \textcolor{gray}{16} & \textcolor{gray}{21} \\
VREx & 67.0 \scriptsize{$\pm$ 1.3} & 97.9 \scriptsize{$\pm$ 0.1} & 78.1 \scriptsize{$\pm$ 0.2} & 87.2 \scriptsize{$\pm$ 0.6} & \textcolor{gray}{65.7} \scriptsize{$\pm$ 0.3} & \textcolor{gray}{51.4} \scriptsize{$\pm$ 0.5} & \textcolor{gray}{30.1} \scriptsize{$\pm$ 3.7} & \textcolor{gray}{68.2} & 10.6 & 8 & 20 \\
RSC & 58.5 \scriptsize{$\pm$ 0.5} & \textcolor{gray}{97.6} \scriptsize{$\pm$ 0.1} & 77.8 \scriptsize{$\pm$ 0.6} & \textcolor{gray}{86.2} \scriptsize{$\pm$ 0.5} & 66.5 \scriptsize{$\pm$ 0.6} & \textcolor{gray}{52.1} \scriptsize{$\pm$ 0.2} & \textcolor{gray}{38.9} \scriptsize{$\pm$ 0.6} & \textcolor{gray}{68.2} & \textcolor{gray}{13.4} & \textcolor{gray}{13} & 19 \\
AND-mask & 58.6 \scriptsize{$\pm$ 0.4} & \textcolor{gray}{97.5} \scriptsize{$\pm$ 0.0} & \textcolor{gray}{76.4} \scriptsize{$\pm$ 0.4} & \textcolor{gray}{86.4} \scriptsize{$\pm$ 0.4} & \textcolor{gray}{66.1} \scriptsize{$\pm$ 0.2} & \textcolor{gray}{49.8} \scriptsize{$\pm$ 0.4} & \textcolor{gray}{37.9} \scriptsize{$\pm$ 0.6} & \textcolor{gray}{67.5} & \textcolor{gray}{17.0} & \textcolor{gray}{16} & \textcolor{gray}{22} \\
SAND-mask & 62.3 \scriptsize{$\pm$ 1.0} & \textcolor{gray}{97.4} \scriptsize{$\pm$ 0.1} & \textcolor{gray}{76.2} \scriptsize{$\pm$ 0.5} & \textcolor{gray}{85.9} \scriptsize{$\pm$ 0.4} & \textcolor{gray}{65.9} \scriptsize{$\pm$ 0.5} & \textcolor{gray}{50.2} \scriptsize{$\pm$ 0.1} & \textcolor{gray}{32.2} \scriptsize{$\pm$ 0.6} & \textcolor{gray}{67.2} & \textcolor{gray}{17.9} & \textcolor{gray}{19} & \textcolor{gray}{22} \\
Fish & 61.8 \scriptsize{$\pm$ 0.8} & 97.9 \scriptsize{$\pm$ 0.1} & 77.8 \scriptsize{$\pm$ 0.6} & \textcolor{gray}{85.8} \scriptsize{$\pm$ 0.6} & \textcolor{gray}{66.0} \scriptsize{$\pm$ 2.9} & \textcolor{gray}{50.8} \scriptsize{$\pm$ 0.4} & \textbf{43.4} \scriptsize{$\pm$ 0.3} & 69.1 & 11.3 & 11 & 18 \\
Fishr & \underline{68.8} \scriptsize{$\pm$ 1.4} & 97.8 \scriptsize{$\pm$ 0.1} & 78.2 \scriptsize{$\pm$ 0.2} & 86.9 \scriptsize{$\pm$ 0.2} & 68.2 \scriptsize{$\pm$ 0.2} & \underline{53.6} \scriptsize{$\pm$ 0.4} & 41.8 \scriptsize{$\pm$ 0.2} & \underline{70.8} & 5.4 & \textbf{3} & 16 \\
SelfReg & 58.0 \scriptsize{$\pm$ 0.7} & \textbf{98.1} \scriptsize{$\pm$ 0.7} & 78.2 \scriptsize{$\pm$ 0.1} & \textbf{87.7} \scriptsize{$\pm$ 0.1} & 68.1 \scriptsize{$\pm$ 0.3} & \textcolor{gray}{52.8} \scriptsize{$\pm$ 0.9} & \underline{43.1} \scriptsize{$\pm$ 0.1} & 69.4 & \underline{5.0} & \textbf{3} & 19 \\
CausIRL\scriptsize{CORAL} & 58.4 \scriptsize{$\pm$ 0.3} & 98.0 \scriptsize{$\pm$ 0.1} & 78.2 \scriptsize{$\pm$ 0.1} & \underline{87.6} \scriptsize{$\pm$ 0.1} & 67.7 \scriptsize{$\pm$ 0.2} & 53.4 \scriptsize{$\pm$ 0.4} & 42.1 \scriptsize{$\pm$ 0.1} & 69.4 & \underline{5.0} & \textbf{3} & 15 \\
CausIRL\scriptsize{MMD} & 63.7 \scriptsize{$\pm$ 0.8} & 97.9 \scriptsize{$\pm$ 0.1} & 78.1 \scriptsize{$\pm$ 0.1} & \textcolor{gray}{86.6} \scriptsize{$\pm$ 0.7} & \textcolor{gray}{65.2} \scriptsize{$\pm$ 0.6} & \textcolor{gray}{52.2} \scriptsize{$\pm$ 0.3} & \textcolor{gray}{40.6} \scriptsize{$\pm$ 0.2} & 69.2 & 10.4 & 10 & 20 \\
\midrule
IDM & \textbf{72.0} \scriptsize{$\pm$ 1.0} & 98.0 \scriptsize{$\pm$ 0.1} & 78.1 \scriptsize{$\pm$ 0.4} & \underline{87.6} \scriptsize{$\pm$ 0.3} & \underline{68.3} \scriptsize{$\pm$ 0.2} & \textcolor{gray}{52.8} \scriptsize{$\pm$ 0.5} & 41.8 \scriptsize{$\pm$ 0.2} & \textbf{71.2} & \textbf{3.3} & \textbf{3} & \textbf{6} \\
\bottomrule
\end{tabular}}
\end{table*}

\section{Experimental Results} \label{sec:expr}
In this section, we evaluate the proposed IDM algorithm on the Colored MNIST task \cite{irm2019} and the DomainBed benchmark \cite{domainbed2020} to demonstrate its capability of generalizing against various types of distribution shifts\footnote{The source code is available at \url{https://github.com/Yuxin-Dong/IDM}.}. Detailed settings of these experiments and further empirical results including ablation studies are reported in Appendix \ref{sec:setting} and \ref{sec:results}.

\subsection{Colored MNIST}
The Colored MNIST task \cite{irm2019} is carefully designed to create high correlations between image colors and the true labels, leading to spurious features that possess superior predictive power ($90\%$ and $80\%$ accuracy) over the actual digits ($75\%$). However, this correlation is reversed in the test domain ($10\%$), causing any learning algorithm that solely minimizes training errors to overfit the color information and fail when testing. As such, Colored MNIST is an ideal task to evaluate the capability of learning algorithms to achieve invariance across source domains.

Following the settings of \cite{irm2019}, we adopt a two-stage training technique, where the penalty strength $\lambda$ is set low initially and higher afterward. We visualize the learning dynamics of relevant DG penalties, including IRM, V-Rex, IGA, and Fishr, using the IDM objective for optimization in Figure \ref{fig:cmnist}. The penalty values are normalized for better clarity. This visualization confirms Theorem \ref{thm:pp_domain} that IDM promotes source-domain generalization by minimizing the gap between training risks, thus ensuring the optimality of the predictor across different training domains. Moreover, it verifies the superiority of PDM by showing that penalizing the IDM objective solely is sufficient to minimize other types of invariance penalties.

Table \ref{tbl:cmnist} presents the performance comparison on Colored MNIST across $10$ independent runs. Following the hyper-parameter tuning technique as \cite{irm2019}, we select the best model by $\max_w \min(L_s(w), L_t(w))$. As can be seen, IDM achieves the best trade-off between training and test-domain accuracies (70.2\%), and near-optimal gray-scale accuracy (70.5\%) compared to the Oracle predictor (71.0\%, ERM trained with gray-scale images).

\subsection{DomainBed Benchmark}
The DomainBed Benchmark \cite{domainbed2020} comprises multiple synthetic and real-world datasets for assessing the performance of both DA and DG algorithms. To ensure a fair comparison, DomainBed limits the number of attempts for hyper-parameter tuning to $20$, and the results are averaged over $3$ independent trials. Therefore, DomainBed serves as a rigorous and comprehensive benchmark to evaluate different DG strategies. We compare the performance of our method with $20$ baselines in total for a thorough evaluation. Table \ref{tbl:domainbed_oracle} summarizes the results using test-domain model selection, which is a common choice for validation purposes \cite{fishr2022, vrex2021} and highly motivated by our discussion in Appendix \ref{sec:full_domainbed}.

As can be seen, IDM achieves top-1 accuracy (72.0\%) on CMNIST which is competitive with the Oracle (75.0\%), outperforming all previous distribution alignment techniques by aligning the directions (AND-mask, SAND-mask, Fish) or low-order moments (Fishr). This verifies the superiority of the proposed PDM method as well as the complementary relationship between gradient and representation alignment. On the contrary, algorithms that only align the representations (CORAL, MMD, DANN, CDANN) are incapable of addressing the concept shift, thus performing poorly on CMNIST. Moreover, IDM achieves the highest accuracy among all distribution matching algorithms on RMNIST / PACS, competitive performances to the best algorithm on RMNIST (98.0\% v.s.~98.1\%), PACS (87.6\% v.s.~87.7\%), OfficeHome (68.3\% v.s.~68.4\%), the highest average accuracy (71.2\%) and best rankings (mean, median and worst rankings on $7$ datasets) among all baseline methods. IDM also enables efficient computation, such that the running-time overhead is only $5\%$ compared to ERM on the largest DomainNet dataset, and negligible for other smaller datasets. Notably, IDM is the only algorithm that consistently achieves top rankings (Top 6 of 22), while any other method failed to outperform most of the competitors on at least $1$ dataset.

While the overall performance is promising, we notice that IDM is not very effective on TerraIncognita. There are several possible reasons: Firstly, the number of hyper-parameters in IDM exceeds most competing methods, which is critical to model selection since the number of tuning attempts is limited in DomainBed. Recall that the value of $\lambda_1$ and $\lambda_2$ should adapt to the amount of covariate and concept shifts respectively: While CMNIST manually induces high concept shift, covariate shift is instead dominant in other datasets, raising extra challenges for hyper-parameter tuning. Secondly, representation space distribution alignment may not always help since $L_t(w) \le L(w)$ is possible by the randomized nature of target domains. These factors together result in sub-optimal hyper-parameter selection results.

\section{Conclusion}
In this work, we explore a novel perspective for DG by minimizing the domain-level generalization gap with high probability, which facilitates information-theoretic analysis for the generalization behavior of learning algorithms. Our analysis sheds light on understanding how gradient or representation matching enhances generalization and unveils the complementary relationship between these two elements. These theoretical insights inspire us to design the IDM algorithm by simultaneously aligning inter-domain gradients and representations, which then achieves superior performance on the DomainBed benchmark.

\section*{Acknowledgments}
This work has been supported by the Key Research and Development Program of China under Grant 2021ZD0110700; The National Natural Science Foundation of China under Grant 62106191;

\appendix

\section*{Prerequisite Definitions and Lemmas}

\begin{definition}{(Subgaussian)}
    A random variable $X$ is $\sigma$-subgaussian if for any $\rho \in \mathbb{R}$, $\E[\exp(\rho(X - \E[X]))] \le \exp(\rho^2\sigma^2/2)$.
\end{definition}

\begin{definition}{(Kullback-Leibler Divergence)}
    Let $P$ and $Q$ be probability measures on the same space $\mathcal{X}$, the KL divergence from $P$ to $Q$ is defined as $\kl{P}{Q} \triangleq \int_{\mathcal{X}} p(x) \log(p(x)/q(x)) \dif x$.
\end{definition}

\begin{definition}{(Mutual Information)}
    Let $(X,Y)$ be a pair of random variables with values over the space $\mathcal{X} \times \mathcal{Y}$. Let their joint distribution be $P_{X,Y}$ and the marginal distributions be $P_X$ and $P_Y$ respectively, the mutual information between $X$ and $Y$ is defined as $I(X;Y) = \kl{P_{X,Y}}{P_X P_Y}$.
\end{definition}

\begin{definition}{(Wasserstein Distance)}
    Let $c(\cdot, \cdot)$ be a metric and let $P$ and $Q$ be probability measures on $\mathcal{X}$. Denote $\Gamma(P,Q)$ as the set of all couplings of $P$ and $Q$ (i.e. the set of all joint distributions on $\mathcal{X} \times \mathcal{X}$ with two marginals being $P$ and $Q$), then the Wasserstein distance of order $p$ between $P$ and $Q$ is defined as $\W_p(P,Q) \triangleq \prn*{\inf_{\gamma \in \Gamma(P,Q)} \int_{\mathcal{X} \times \mathcal{X}} c(x,x')^p \dif \gamma(x,x')}^{1/p}$.
\end{definition}

Unless otherwise noted, we use $\log$ to denote the logarithmic function with base $e$, and use $\W(\cdot, \cdot)$ to denote the Wasserstein distance of order $1$.

\begin{definition}{(Total Variation)}
    The total variation between two probability measures $P$ and $Q$ is $\TV(P,Q) \triangleq \sup_E \abs{P(E) - Q(E)}$, where the supremum is over all measurable set $E$.
\end{definition}

\begin{lemma}{(Lemma 1 in \cite{harutyunyan2021information})} \label{lm:kl_con_pair}
    Let $(X,Y)$ be a pair of random variables with joint distribution $P_{X,Y}$ and let $\bar{Y}$ be an independent copy of $Y$. If $f(x,y)$ is a measurable function such that $E_{X,Y}[f(X,Y)]$ exists and $f(X,\bar{Y})$ is $\sigma$-subgaussian, then
    \begin{equation*}
        \abs*{\E_{X,Y}[f(X,Y)] - \E_{X,\bar{Y}}[f(X,\bar{Y})]} \le \sqrt{2\sigma^2 I(X;Y)}.
    \end{equation*}
    Furthermore, if $f(x,Y)$ is $\sigma$-subgaussian for each $x$ and the expectation below exists, then
    \begin{equation*}
        \E_{X,Y}\brk*{\prn*{f(X,Y) - \E_{\bar{Y}}[f(X,\bar{Y})]}^2} \le 4\sigma^2(I(X;Y) + \log 3),
    \end{equation*}
    and for any $\epsilon > 0$, we have
    \begin{equation*}
        \Pr\brc*{\abs*{f(X,Y) - \E_{\bar{Y}}[f(X,\bar{Y})]} \ge \epsilon} \le \frac{4\sigma^2(I(X;Y) + \log 3)}{\epsilon^2}.
    \end{equation*}
\end{lemma}

\begin{lemma}{(Lemma 2 in \cite{harutyunyan2021information})} \label{lm:sg_sq}
    Let $X$ be $\sigma$-subgaussian and $\E[X] = 0$, then for any $\lambda \in [0, 1/4\sigma^2)$:
    \begin{equation*}
        \E_X \brk*{e^{\lambda X^2}} \le 1 + 8\lambda \sigma^2.
    \end{equation*}
\end{lemma}

\begin{lemma}{(Donsker-Varadhan formula)} \label{lm:dv_ineq} 
    Let $P$ and $Q$ be probability measures defined on the same measurable space, where $P$ is absolutely continuous with respect to $Q$. Then
    \begin{equation*}
        \kl{P}{Q} = \sup_X \brc*{\E_P[X] - \log \E_Q[e^X]},
    \end{equation*}
    where $X$ is any random variable such that $e^X$ is $Q$-integrable and $\E_P[X]$ exists.
\end{lemma}

\begin{lemma} \label{lm:kl_con}
    Let $P$, and $Q$ be probability measures defined on the same measurable space. Let $X \sim P$ and $X' \sim Q$. If $f(X)$ is $\sigma$-subgaussian w.r.t $X$ and the following expectations exists, then
    \begin{gather*}
        \abs*{\E_{X'}[f(X')] - \E_X[f(X)]} \le \sqrt{2\sigma^2 \kl{Q}{P}}, \\
        \E_{X'}\brk*{\prn*{f(X') - \E_X[f(X)]}^2} \le 4\sigma^2 (\kl{Q}{P} + \log 3).
    \end{gather*}
    Furthermore, by combining the results above and Markov's inequality, we have that for any $\epsilon > 0$:
    \begin{equation*}
        \Pr\brc*{\abs{f(X') - \E_X[f(X)]} \ge \epsilon} \le \frac{4\sigma^2}{\epsilon^2} \prn*{\kl{Q}{P} + \log 3}.
    \end{equation*}
\end{lemma}
\begin{proof}
    Let $\lambda \in \mathbb{R}$ be any non-zero constant, then by the subgaussian property of $f(X)$:
    \begin{align*}
        \log \E_X\brk*{e^{\lambda (f(X) - \E_X[f(X)])}} &\le \frac{\lambda^2 \sigma^2}{2}, \\
        \log \E_X\brk*{e^{\lambda f(X)}} - \lambda \E_X[f(X)] &\le \frac{\lambda^2 \sigma^2}{2}.
    \end{align*}
    By applying Lemma \ref{lm:dv_ineq} with $X = \lambda f(X)$ we have
    \begin{align*}
        \kl{Q}{P} &\ge \sup_\lambda \brc*{\E_{X'}[\lambda f(X')] - \log \E_X\brk*{e^{\lambda f(X)}}} \\
        &\ge \sup_\lambda \brc*{\E_{X'}[\lambda f(X')] - \lambda \E_X[f(X)] - \frac{\lambda^2 \sigma^2}{2}} \\
        &= \frac{1}{2\sigma^2} \prn*{\E_{X'}[f(X')] - \E_X[f(X)]}^2,
    \end{align*}
    where the supremum is taken by setting $\lambda = \frac{1}{\sigma^2} (\E_{X'}[f(X')] - \E_X[f(X)])$. This completes the proof of the first inequality.

    To prove the second inequality, let $g(x) = (f(x) - \E_X[f(X)])^2$ and $\lambda \in [0, 1/4\sigma^2)$. Apply Lemma \ref{lm:dv_ineq} again with $X = \lambda g(X)$, we have
    \begin{align*}
        \kl{Q}{P} &\ge \sup_\lambda \brc*{\E_{X'}[\lambda g(X')] - \log \E_X\brk*{e^{\lambda g(X)}}} \\
        &= \sup_\lambda \brc*{\E_{X'}\brk*{\lambda \prn*{f(X') - \E_X[f(X)]}^2} - \log \E_X\brk*{e^{\lambda \prn*{f(X) - \E_X[f(X)]}^2}}} \\
        &\ge \sup_\lambda \brc*{\E_{X'}\brk*{\lambda \prn*{f(X') - \E_X[f(X)]}^2} - \log(1 + 8\lambda \sigma^2)} \\
        &\ge \frac{1}{4\sigma^2} \E_{X'}\brk*{\prn*{f(X') - \E_X[f(X)]}^2} - \log 3,
    \end{align*}
    where the second inequality follows by applying Lemma \ref{lm:sg_sq} and the last inequality follows by taking $\lambda \rightarrow \frac{1}{4\sigma^2}$. This finishes the proof of the second inequality.
    
    Furthermore, by applying Markov's inequality, we can get:
    \begin{align*}
        \Pr\brc*{\abs{f(X') - \E_X[f(X)]} \ge \epsilon} &= \Pr\brc*{\prn*{f(X') - \E_X[f(X)]}^2 \ge \epsilon^2} \\
        &\le \frac{1}{\epsilon^2} \E_{X'}\brk*{\prn*{f(X') - \E_X[f(X)]}^2} \\
        &\le \frac{4\sigma^2}{\epsilon^2} \prn*{\kl{Q}{P} + \log 3},
    \end{align*}
    which completes the proof.
\end{proof}

\begin{lemma}{(Kantorovich-Rubinstein Duality)} \label{lm:kr_dual}
    Let $P$ and $Q$ be probability measures defined on the same measurable space $\mathcal{X}$, then
    \begin{equation*}
        \W(P,Q) = \sup_{f \in Lip_1} \brc*{\int_\mathcal{X} f \dif P - \int_\mathcal{X} f \dif Q},
    \end{equation*}
    where $Lip_1$ denotes the set of $1$-Lipschitz functions in the metric $c$, i.e. $\abs{f(x) - f(x')} \le c(x,x')$ for any $f \in Lip_1$ and $x, x' \in \mathcal{X}$.
\end{lemma}

\begin{lemma}{(Pinsker's Inequality)} \label{lm:pinsker}
    Let $P$ and $Q$ be probability measures defined on the same space, then $\TV(P, Q) \le \sqrt{\frac{1}{2} \kl{Q}{P}}$.
\end{lemma}

\begin{proposition}
    For any constant $\lambda \in (0, 1)$, we have
    \begin{align*}
        \Pr\{\abs{L_s(W) - L_t(W)} \ge \epsilon\} &\le \Pr\{\abs{L_s(W) - L(W)} \ge \lambda\epsilon\} \\
        &\quad + \Pr\{\abs{L_t(W) - L(W)} \ge (1-\lambda)\epsilon\}.
    \end{align*}
\end{proposition}
\begin{proof}
    Notice that $\abs{L_s(W) - L(W)} \le \lambda\epsilon$ and $\abs{L_t(W) - L(W)} \le (1-\lambda)\epsilon$ together implies $\abs{L_s(W) - L_t(W)} \le \epsilon$, we then have
    \begin{equation*}
        \Pr\{\abs{L_s(W) - L_t(W)} \le \epsilon\} \ge \Pr\{\abs{L_s(W) - L(W)} \le \lambda\epsilon \bigcap \abs{L_t(W) - L(W)} \le (1-\lambda)\epsilon\}.
    \end{equation*}
    This implies that
    \begin{equation*}
        \Pr\{\abs{L_s(W) - L_t(W)} \ge \epsilon\} \le \Pr\{\abs{L_s(W) - L(W)} \ge \lambda\epsilon \bigcup \abs{L_t(W) - L(W)} \ge (1-\lambda)\epsilon\}.
    \end{equation*}
    By applying Boole's inequality, we then have
    \begin{align*}
        \Pr\{\abs{L_s(W) - L_t(W)} \ge \epsilon\} &\le \Pr\{\abs{L_s(W) - L(W)} \ge \lambda\epsilon\} \\
        &\quad + \Pr\{\abs{L_t(W) - L(W)} \ge (1-\lambda)\epsilon\}.
    \end{align*}
\end{proof}

\section*{Omitted Proofs in Section \ref{sec:theory}} \label{sec:proof1}

\subsection{Proof of Proposition \ref{prop:concept_shift}}

\begin{restateproposition}{\ref{prop:concept_shift}}
    Let $Q_{R|X}$ and $Q_{Y|R}$ denote the encoder and classifier characterized by $f_\phi$ and $f_\psi$ respectively, and let $Q_{Y|X}$ denote the whole model by $Q_{Y|X} = \int_\mathcal{R} Q_{Y|R} \dif Q_{R|X}$. Then for any domain $d \in \mathcal{D}$:
    \begin{gather*}
        \kl{P_{Y|X,D}}{Q_{Y|X}} \ge I(Y;D|X), \\
        \kl{P_{Y|X,D=d}}{Q_{Y|X}} \le I(X;Y|R,D=d) + \kl{P_{Y|R,D=d}}{Q_{Y|R}}.
    \end{gather*}
\end{restateproposition}
\begin{proof}
    For better clarity, we abbreviate $P_X(X)$ as $P_X$ in the following expectations for any random variable $X$.
    \begin{align*}
        \kl{P_{Y|X,D}}{Q_{Y|X}} &= \E_{D,X,Y} \brk*{\log \frac{P_{Y|X,D}}{Q_{Y|X}}} \\
        &= \E_{D,X,Y} \brk*{\log \frac{P_{Y|X,D}}{P_{Y|X}} \cdot \frac{P_{Y|X}}{Q_{Y|X}}} \\
        &= \E_{D,X,Y} \brk*{\log \frac{P_{Y,D|X}}{P_{Y|X}P_{D|X}}} + \E_{X,Y} \brk*{\log \frac{P_{Y|X}}{Q_{Y|X}}} \\
        &= I(Y;D|X) + \kl{P_{Y|X}}{Q_{Y|X}} \ge I(Y;D|X).
    \end{align*}
    The last inequality is by the positiveness of the KL divergence. It holds with equality if and only if $Q_{Y|X} = P_{Y|X}$.
    
    To prove the second inequality, we apply Jensen's inequality on the concave logarithmic function:
    \begin{align*}
        \kl{P_{Y|X,D=d}}{Q_{Y|X}} &= \E_{X,Y|D=d} \brk*{\log \frac{P_{Y|X,D=d}}{Q_{Y|X}}} \\
        &= \E_{X,Y|D=d} \brk*{\log \frac{P_{Y|X,D=d}}{\E_{R|X=x} [Q_{Y|R}]}} \\
        &\le \E_{X,Y|D=d} \E_{R|X=x} \brk*{\log \frac{P_{Y|X,D=d}}{Q_{Y|R}}} \\
        &= \E_{X,Y,R|D=d} \brk*{\log \frac{P_{Y|X,D=d}}{P_{Y|R,D=d}} \cdot \frac{P_{Y|R,D=d}}{Q_{Y|R}}} \\
        &= \E_{X,Y,R|D=d} \brk*{\log \frac{P_{Y,X|R,D=d}}{P_{Y|R,D=d}P_{X|R,D=d}}} + \E_{Y,R|D=d} \brk*{\log \frac{P_{Y|R,D=d}}{Q_{Y|R}}} \\
        &= I(X;Y|R,D=d) + \kl{P_{Y|R,D=d}}{Q_{Y|R}}.
    \end{align*}
    The only inequality holds with equality when $\Var[Q_{R|X=x}] = 0$ for any $x \in \mathcal{X}$, i.e. $f_\phi$ is deterministic. This completes the proof.
\end{proof}

\subsection{Proof of Theorem \ref{thm:pp_domain}}

\begin{restatetheorem}{\ref{thm:pp_domain}}
    If Assumption \ref{asmp:loss_bounded} holds, then
    \begin{gather*}
        \abs*{\E_{W,D_s}[L_s(W)] - \E_{W}[L(W)]} \le \frac{1}{m} \sum_{i=1}^m \sqrt{\frac{M^2}{2} I(W,D_i)}, \\
        \Pr\brc*{\abs*{L_s(W) - L(W)} \ge \epsilon} \le \frac{M}{m\epsilon\sqrt{2}} \sum_{i=1}^m \sqrt{I(W,D_i)} + \frac{1}{\epsilon}\E_{W,D}\abs{L_D(W) - L(W)},
    \end{gather*}
    where $D \sim \nu$ is independent of $W$.
\end{restatetheorem}
\begin{proof}
    For any $D \in D_s$, let $\bar{D}$ be an independent copy of the marginal distribution of $D$. Since $\ell(\cdot, \cdot) \in [0, M]$, we know that $L_{\bar{D}}(W)$ is $\frac{M}{2}$-subgaussian. Then by setting $X = W$, $Y = D$ and $f(W,D) = L_D(W)$ in Lemma \ref{lm:kl_con_pair}, we have
    \begin{align*}
        \abs{\E_{W,D}[L_D(W)] - \E_W[L(W)]} &= \abs{\E_{W,D}[L_D(W)] - \E_{W,\bar{D}}[L_{\bar{D}}(W)]} \le \sqrt{\frac{M^2}{2} I(W,D)}.
    \end{align*}
    By summing up the inequality above over each training domain, we can get
    \begin{align*}
        \abs{\E_{W,D_s}[L_s(W)] - \E_W[L(W)]} &= \abs*{\frac{1}{m} \E_{W,D_s}\brk*{\sum_{i=1}^m L_{D_i}(W)} - \E_W[L(W)]} \\
        &\le \frac{1}{m} \sum_{i=1}^m \abs*{\E_{W,D_i}[L_{D_i}(W)] - \E_W[L(W)]} \\
        &\le \frac{1}{m} \sum_{i=1}^m \sqrt{\frac{M^2}{2} I(W,D_i)}.
    \end{align*}
    Similarly, for any $D \in D_s$, one can verify that $\abs{L_{\bar{D}}(W) - L(W)} \in [0, M]$ and is thus $\frac{M}{2}$-subgaussian. By applying Lemma \ref{lm:kl_con_pair} with $X = W$, $Y = D$ and $f(W, D) = \abs{L_D(W) - L(W)}$, we obtain
    \begin{align*}
        \E_{W,D}\abs{L_D(W) - L(W)} - \E_{W,\bar{D}}\abs{L_{\bar{D}}(W) - L(W)} \le \sqrt{\frac{M^2}{2} I(W,D)}.
    \end{align*}
    Summing up this inequality over each training domain, we then have
    \begin{align*}
        \E_{W,D_s}\abs{L_s(W) - L(W)} &= \E_{W,D_s}\abs*{\frac{1}{m} \sum_{i=1}^m L_{D_i}(W) - L(W)} \\
        &\le \frac{1}{m} \sum_{i=1}^m \E_{W,D_i}\abs{L_{D_i}(W) - L(W)} \\
        &\le \frac{1}{m} \sum_{i=1}^m \sqrt{\frac{M^2}{2} I(W,D_i)} + \E_{W,\bar{D}}\abs{L_{\bar{D}}(W) - L(W)}.
    \end{align*}
    By applying Markov's inequality, we finally have
    \begin{align*}
        \Pr\brc*{\abs{L_s(W) - L(W)} \ge \epsilon} \le \frac{M}{m\epsilon\sqrt{2}} \sum_{i=1}^m \sqrt{I(W,D_i)} + \frac{1}{\epsilon}\E_{W,\bar{D}}\abs{L_{\bar{D}}(W) - L(W)}.
    \end{align*}
    
    Additionally, by assuming that the training domains are independent, we have
    \begin{align*}
        I(W;D_s) &= I(W;\{D_i\}_{i=1}^m) = I(W;D_1) + I(W;\{D_i\}_{i=2}^m|D_1) \\
        &= I(W;D_1) + I(W;\{D_i\}_{i=2}^m) - I(\{D_i\}_{i=2}^m;D_1) + I(\{D_i\}_{i=2}^m;D_1|W) \\
        &= I(W;D_1) + I(W;\{D_i\}_{i=2}^m) + I(\{D_i\}_{i=2}^m;D_1|W) \\
        &\ge I(W;D_1) + I(W;\{D_i\}_{i=2}^m) \\
        &\ge \cdots \\
        &\ge \sum_{i=1}^m I(W,D_i).
    \end{align*}
\end{proof}

\subsection{Proof of Theorem \ref{thm:pp_domain_wass}}

\begin{restatetheorem}{\ref{thm:pp_domain_wass}}
    Let $W$ be the output of learning algorithm $\mathcal{A}$ under training domains $D_s$. If $\ell(f_w(X),Y)$ is $\beta'$-Lipschitz w.r.t $w$, i.e. $\abs{\ell(f_{w_1}(X),Y) - \ell(f_{w_2}(X),Y)} \le \beta' c(w_1,w_2)$ for any $w_1, w_2 \in \mathcal{W}$, then
    \begin{equation*}
        \abs{\E_{W,D_s}[L_s(W)] - \E_W[L(W)]} \le \frac{\beta'}{m} \sum_{i=1}^m \E_{D_i} [\W(P_{W|D_i}, P_W)].
    \end{equation*}
\end{restatetheorem}
\begin{proof}
    For any $D_i \in D_s$, let $P = P_{W|D_i=d}$, $Q = P_W$ and $f(w) = L_{D_i}(w)$ in Lemma \ref{lm:kr_dual}, then
    \begin{align*}
        \abs{\E_{W,D_s}[L_s(W)] - \E_W[L(W)]} &\le \frac{1}{m} \E_{D_s}\brk*{ \sum_{i=1}^m \abs*{\E_{W|D_i}\brk*{L_{D_i}(W)} - \E_W[L(W)]}} \\
        &= \frac{1}{m} \E_{D_s}\brk*{\sum_{i=1}^m \abs{\E_{W|D_i}[L_{D_i}(W)] - \E_W[L_{D_i}(W)]}} \\
        &\le \frac{1}{m} \E_{D_s}\brk*{\sum_{i=1}^m \beta' \W(P_{W|D_i}, P_W)} \\
        &= \frac{\beta'}{m} \sum_{i=1}^m \E_{D_i} [\W(P_{W|D_i}, P_W)].
    \end{align*}
    When the metric $d$ is discrete, the Wasserstein distance is equal to the total variation. Combining with Lemma \ref{lm:pinsker}, we have the following reductions:
    \begin{align*}
        \E_{D_i}[\W(P_{W|D_i}, P_W)] &= \E_{D_i}[\TV(P_{W|D_i}, P_W)] \\
        &\le \E_{D_i}\brk*{\sqrt{\frac{1}{2} \kl{P_{W|D_i}}{P_W}}} \\
        &\le \sqrt{\frac{1}{2} I(W;D_i)},
    \end{align*}
    where the last inequality follows by applying Jensen's inequality on the concave square root function.
\end{proof}

\subsection{Proof of Theorem \ref{thm:pp_con}}

\begin{restatetheorem}{\ref{thm:pp_con}}
    For any $w \in \mathcal{W}$, $\E_{D_t}[L_t(w)] = L(w)$. Additionally if Assumption \ref{asmp:loss_subgauss} holds, then
    \begin{equation*}
        \Pr\brc*{\abs{L_t(w) - L(w)} \ge \epsilon} \le \frac{\sigma}{\epsilon}\sqrt{2 I(Z;D)}.
    \end{equation*}
\end{restatetheorem}
\begin{proof}
    By the identical marginal distribution of the test domains $\mathcal{D}_t = \{D_k\}_{k=1}^{m'}$, we have
    \begin{align*}
        \E_{D_t}[L_t(w)] &= \frac{1}{m'} \sum_{k=1}^{m'} \E_{D_k}[L_{D_k}(w)] = \frac{1}{m'} \sum_{k=1}^{m'} \E_D[L_D(w)] \\
        &= \E_D[L_D(w)] = L(w).
    \end{align*}
    For any $d \in \mathcal{D}$, by applying Lemma \ref{lm:kl_con} with $P = P_Z$, $Q = P_{Z|D=d}$ and $f(Z) = \ell(f_w(X),Y)$, we can get
    \begin{align*}
        \abs{L_d(w) - L(w)} &= \abs*{\E_{Z|D=d}[\ell(f_w(X),Y)] - \E_Z[\ell(f_w(X),Y]} \\
        &\le \sqrt{2\sigma^2 \kl{P_{Z|D=d}}{P_Z}}.
    \end{align*}
    Taking the expectation over $D \sim \nu$, we can get
    \begin{align*}
        \E_D\abs{L_d(w) - L(w)} &\le \E_D \sqrt{2\sigma^2 \kl{P_{Z|D}}{P_Z}} \\
        &\le \sqrt{2\sigma^2 \E_D[\kl{P_{Z|D}}{P_Z}]} \\
        &= \sigma\sqrt{2 I(Z;D)}.
    \end{align*}
    By summing up the inequality above over each test domain, we can get
    \begin{align*}
        \E_{D_t}\abs{L_t(w) - L(w)} &= \E_{D_t}\abs*{\frac{1}{m'}\sum_{k=1}^{m'} L_{D_k}(w) - L(W)} \\
        &\le \frac{1}{m'}\sum_{k=1}^{m'} \E_{D_k}\abs*{L_{D_k}(w) - L(W)} \\
        &\le \frac{1}{m'}\sum_{k=1}^{m'} \sigma\sqrt{2 I(Z;D_k)} \\
        &= \sigma\sqrt{2 I(Z;D)}.
    \end{align*}
    By applying Markov's inequality, we finally have
    \begin{align*}
        \Pr\brc*{\abs{L_t(w) - L(w)} \ge \epsilon} \le \frac{\sigma}{\epsilon}\sqrt{2 I(Z;D)}.
    \end{align*}
\end{proof}

\subsection{Proof of Theorem \ref{thm:pp_best}}

\begin{restatetheorem}{\ref{thm:pp_best}}
    If Assumption \ref{asmp:loss_bounded} and \ref{asmp:loss_dist} hold, then for any $w = (\phi, \psi) \in \mathcal{W}$,
    \begin{equation*}
        \Pr\brc*{L_t(\psi) - L(\psi) \ge \epsilon} \le \frac{\sigma}{\epsilon} \sqrt{2 I(R;D)} + \frac{2}{\epsilon} L^*.
    \end{equation*}
    where $L^* = \min_{f^*: \mathcal{R} \mapsto \mathcal{Y}} L(f^*)$.
\end{restatetheorem}
\begin{proof}
    For any environment $d \in \mathcal{D}$, classifier $\psi$ and $f^*: \mathcal{R} \mapsto \mathcal{Y}$, denote
    \begin{gather*}
        L_d(\psi, f^*) = \E_{R|D=d}[\ell(f_\psi(R), f^*(R))], \\
        L(\psi, f^*) = \E_R[\ell(f_\psi(R), f^*(R))].
    \end{gather*}
    By setting $P = P_R$, $Q = P_{R|D=d}$ and $f(R) = \ell(f_\psi(R), f^*(R))$ and applying Lemma \ref{lm:kl_con}, we have
    \begin{equation*}
        \abs*{L_d(\psi, f^*) - L(\psi, f^*)} \le \sqrt{2\sigma^2 \kl{P_{R|D=d}}{P_R}}.
    \end{equation*}
    By taking the expectation over $D \sim \nu$, we get
    \begin{align*}
        \E_D \abs*{L_D(\psi, f^*) - L(\psi, f^*)} &\le \E_D\sqrt{2\sigma^2 \kl{P_{R|D}}{P_R}} \\
        &\le \sqrt{2\sigma^2 I(R;D)}.
    \end{align*}
    If Assumption \ref{asmp:loss_dist} holds, then by the symmetry and triangle inequality of $\ell(\cdot, \cdot)$, we have
    \begin{align*}
        L_d(\psi, f^*) &= \E_{R|D=d}[\ell(f_\psi(R), f^*(R))] \\
        &\le \E_{R,Y|D=d}[\ell(f_\psi(R), Y) + \ell(f^*(R), Y)] \\
        &= L_d(\psi) + L_d(f^*). \\
        L(\psi, f^*) &\le L(\psi) + L(f^*).
    \end{align*}
    Similarly, we can prove that
    \begin{align*}
        L_d(\psi, f^*) &= \E_{R|D=d}[\ell(f_\psi(R), f^*(R))] \\
        &\ge \E_{R,Y|D=d}[\ell(f_\psi(R), Y) - \ell(f^*(R), Y)] \\
        &= L_d(\psi) - L_d(f^*). \\
        L(\psi, f^*) &\ge L(\psi) - L(f^*).
    \end{align*}
    Combining the results above, we have
    \begin{align*}
        L_d(\psi) - L(\psi) &\le L_d(\psi, f^*) + L_d(f^*) - L(\psi, f^*) + L(f^*), \\
        L(\psi) - L_d(\psi) &\le L(\psi, f^*) + L(f^*) - L_d(\psi, f^*) + L_d(f^*).
    \end{align*}
    Combining the two inequalities above, we then get
    \begin{equation*}
        \abs*{L_d(\psi) - L(\psi)} \le \abs*{L_d(\psi, f^*) - L(\psi, f^*)} + L_d(f^*) + L(f^*).
    \end{equation*}
    By taking the expectation over $D \sim \nu$, we obtain
    \begin{align*}
        \E_D \abs*{L_D(\psi) - L(\psi)} &\le \E_D \abs*{L_D(\psi, f^*) - L(\psi, f^*)} + \E_D [L_D(f^*) + L(f^*)] \\
        &\le \sqrt{2\sigma^2 I(R;D)} + 2L(f^*).
    \end{align*}
    Summing up the inequality above over each test domain, we can get
    \begin{align*}
        \E_{D_t} \abs*{L_t(\psi) - L(\psi)} &\le \E_{D_t} \abs*{\frac{1}{m'}\sum_{k=1}^{m'} L_{D_k}(\psi) - L(\psi)} \\
        &\le \frac{1}{m'}\sum_{k=1}^{m'} \E_{D_k} \abs*{L_{D_k}(\psi) - L(\psi)} \\
        &\le \frac{1}{m'}\sum_{k=1}^{m'} \prn*{\sqrt{2\sigma^2 I(R;D)} + 2L(f^*)} \\
        &= \sqrt{2\sigma^2 I(R;D)} + 2L(f^*).
    \end{align*}
    By applying Markov's inequality, we finally have
    \begin{align*}
        \Pr\brc*{L_t(\psi) - L(\psi) \ge \epsilon} \le \frac{\sigma}{\epsilon} \sqrt{2 I(R;D)} + \frac{2}{\epsilon} L(f^*).
    \end{align*}
    The proof is complete by taking the minimizer of $\min_{f^*}[L(f^*)]$.
\end{proof}

\subsection{Proof o Proposition \ref{prop:cov_shift_emp}}

\begin{restateproposition}{\ref{prop:cov_shift_emp}}
    Assume that $P_{R,D} \ll P_{R_i,D_i}$ and $P_{R_i,D_i} \ll P_{R,D}$, then
    \begin{equation*}
        \skl{P_{R,D}}{P_{R_i,D_i}} = O(\sqrt{I(W;D_i)}).
    \end{equation*}
\end{restateproposition}
\begin{proof}
    The condition $P_{R,D} \ll P_{R_i,D_i}$ and $P_{R_i,D_i} \ll P_{R,D}$ implies that there exists $B > 1$, such that for any $r \in \mathcal{R}$, $d \in \mathcal{D}$, we have $\frac{P_{R,D}(r,d)}{P_{R_i,D_i}(r,d)} \in [\frac{1}{B}, B]$. Therefore, $\log\frac{P_{R,D}}{P_{R_i,D_i}} \in [-\log(B), \log(B)]$ and is $\log(B)$-subgaussian.

    By applying Lemma \ref{lm:kl_con_pair} with $X = W$, $Y = D_i$ and $f(W, D) = \E_{X|D}[\log\frac{P_{R,D}}{P_{R_i,D_i}}]$, we have
    \begin{align*}
        \abs*{\E_{W,D_i,R_i}\brk*{\log\frac{P_{R,D}(R_i,D_i)}{P_{R_i,D_i}(R_i,D_i)}} - \E_{W,D,R}\brk*{\log\frac{P_{R,D}(R,D)}{P_{R_i,D_i}(R,D)}}} &\le \sqrt{2\log^2(B) I(W;D_i)} \\
        \abs*{\kl{P_{R,D}}{P_{R_i,D_i}} + \kl{P_{R_i,D_i}}{P_{R,D}}} &\le \sqrt{2\log^2(B) I(W;D_i)} \\
        \skl{P_{R,D}}{P_{R_i,D_i}} &\le \log(B)\sqrt{2I(W;D_i)}.
    \end{align*}
\end{proof}

\section*{Omitted Proofs in Section \ref{sec:alg_design}} \label{sec:proof2}

\subsection{Proof of Theorem \ref{thm:mi_bound}}

\begin{restatetheorem}{\ref{thm:mi_bound}}
    Let $G_t^i = -\eta_t g(W_{t-1},B_t^i)$ and $G_t = \sum_{i=1}^m G_t^i$, then
    \begin{equation*}
        I(W_T;D_i) \le \sum_{t=1}^T I(G_t; D_i|W_{t-1}).
    \end{equation*}
    Additionally, if the training domains are independent, then
    \begin{equation*}
        I(W_T;D_i) \le \sum_{t=1}^T I(G_t^i; D_i|W_{t-1}).
    \end{equation*}
\end{restatetheorem}
\begin{proof}
    Noticing the Markov chain relationship $D_i \rightarrow (W_{T-1}, G_T) \rightarrow W_{T-1} + G_T$, then by the data processing inequality
    \begin{align*}
        I(W_T;D_i) &= I\prn*{W_{T-1} + G_T; D_i} \\
        &\le I\prn*{W_{T-1}, G_T; D_i} \\
        &= I(W_{T-1}; D_i) + I\prn*{G_T; D_i| W_{T-1}}.
    \end{align*}
    where the last equality is by the chain rule of conditional mutual information. By applying the reduction steps above recursively, we can get
    \begin{align*}
        I(W_T;D_i) &\le I(W_{T-1}; D_i) + I\prn*{G_T; D_i| W_{T-1}} \\
        &\le I(W_{T-2}; D_i) + I\prn*{G_{T-1}; D_i| W_{T-2}} \\
        &\qquad+ I\prn*{G_T; D_i| W_{T-1}} \\
        &\le \cdots \\
        &\le \sum_{t=1}^T I\prn*{G_t; D_i| W_{t-1}}.
    \end{align*}
    When the training domains are independent, we additionally have
    \begin{align*}
        I(W_T;D_i) &\le I(W_{T-1}; D_i) + I\prn*{G_T; D_i| W_{T-1}} \\
        &\le I(W_{T-1}; D_i) + I\prn*{\{G_T^k\}_{k=1}^m; D_i| W_{T-1}} \\
        &= I(W_{T-1}; D_i) + I\prn*{G_T^i; D_i| W_{T-1}} + I\prn*{\{G_T^k\}_{k=1}^m \setminus G_T^i; D_i| W_{T-1}, G_T^i} \\
        &= I(W_{T-1}; D_i) + I\prn*{G_T^i; D_i| W_{T-1}}.
    \end{align*}
    Then by the same scheme of recursive reduction, we can prove that
    \begin{equation*}
        I(W_T;D_i) \le \sum_{t=1}^T I\prn*{G_t^i; D_i| W_{t-1}}.
    \end{equation*}
\end{proof}



\subsection{Proof of Theorem \ref{thm:indist}}

\begin{restatetheorem}{\ref{thm:indist}}[Formal]
    Let $n$ be the dimensionality of the distribution and $b$ be the number of data points, then
    \begin{itemize}
        \item If $n > b + 1$, then for any sampled data points $s = \{x_i\}_{i=1}^b$, there exists infinite environments $d_1$, $d_2$, $\cdots$ such that $p(S=s|D=d_1) = p(S=s|D=d_2) = \cdots$.
        \item If $n > 2b + 1$, then for any two batch of sampled data points $s_1 = \{x_i^1\}_{i=1}^b$ and $s_2 = \{x_i^2\}_{i=1}^b$, there exists infinite environments $d_1$, $d_2$, $\cdots$ such that for each $j \in [1, \infty)$, $p(S=s_1|D=d_j) = p(S=s_2|D=d_j)$.
    \end{itemize}
\end{restatetheorem}
\begin{proof}
    Without loss of generality, we assume that the data-generating distributions $p(X|D)$ are Gaussian with zero means for simplicity, i.e.
    \begin{equation*}
        p(x|D=d) = \frac{1}{\sqrt{(2\pi)^n \abs{\Sigma_d}}} \exp\prn*{-\frac{1}{2} x^\top \Sigma_d x},
    \end{equation*}
    where $\Sigma_d$ is the corresponding covariance matrix of environment $d$. Let $X \in \mathbb{R}^{n \times b}$ be the data matrix of $S$ such that the $i$-th column of $X$ equals $x_i$, we then have
    \begin{equation*}
        p(S=s|D=d) = \frac{1}{\sqrt{(2\pi)^{bn} \abs{\Sigma_d}^b}} \exp\prn*{-\frac{1}{2} \tr(X^\top \Sigma_d X)}.
    \end{equation*}
    Since the rank of $X$ is at most $b$, one can decompose $\Sigma_d = \Sigma_d^1 + \Sigma_d^2$ with $\rank(\Sigma_d^1) = b$ and $\rank(\Sigma_d^2) = n - b \ge 2$ through eigenvalue decomposition, and let the eigenvector space of $\Sigma_d^1$ cover the column space of $X$. Then we have
    \begin{equation*}
        \tr(X^\top \Sigma_d^1 X) = \tr(X^\top \Sigma_d X), \quad \textrm{and} \quad \tr(X^\top \Sigma_d^2 X) = 0.
    \end{equation*}
    Therefore, one can arbitrarily modify the eigenvector space of $\Sigma_d^2$ as long as keeping it orthogonal to that of $\Sigma_d^1$, without changing the value of $\tr(X^\top \Sigma_d X)$. This finishes the proof of the first part.

    To prove the second part, similarly we decompose $\Sigma_d$ by $\Sigma_d^1 + \Sigma_d^2$ such that $\rank(\Sigma_d^1) = 2b + 1$ and $\rank(\Sigma_d^2) = n - 2b - 1 \ge 1$, and make the eigenvector space of $\Sigma_d^1$ cover the column space of both $X_1$ and $X_2$, where $X_1$ and $X_2$ are the data matrix of $S_1$ and $S_2$ respectively. We then have
    \begin{align*}
        \tr(X_1^\top \Sigma_d^1 X_1) &= \tr(X_1^\top \Sigma_d X_1), \\
        \tr(X_2^\top \Sigma_d^1 X_2) &= \tr(X_2^\top \Sigma_d X_2), \\
        \textrm{and} \quad \tr(X_1^\top \Sigma_d^2 X_1) &= \tr(X_2^\top \Sigma_d^2 X_2) = 0.
    \end{align*}
    Let $\Sigma_d^1 = U_d^\top \Lambda_d U_d$ be the eigenvalue decomposition of $\Sigma_d^1$, where $U_d \in \mathbb{R}^{(2b+1) \times n}$ and $\Lambda_d = \mathrm{diag}(\lambda_1^d, \cdots, \lambda_{2b+1}^d)$. Notice that for any $x \in \mathbb{R}^n$, we have $x^\top \Sigma_d x = (U_d x)^\top \Lambda_d (U_d x) = \sum_{i=1}^{2b+1} (U_d x)_i^2 \lambda_i$. By assuming that $p(S=s_1|D=d) = p(S=s_2|D=d)$, we have the following homogeneous linear equations:
    \begin{align*}
        a_1^1 \lambda_1 + a_1^2 \lambda_2 + \cdots + a_1^{2b+1} \lambda_{2b+1} &= 0, \\
        a_2^1 \lambda_1 + a_2^2 \lambda_2 + \cdots + a_2^{2b+1} \lambda_{2b+1} &= 0, \\
        &\cdots \\
        a_b^1 \lambda_1 + a_b^2 \lambda_2 + \cdots + a_b^{2b+1} \lambda_{2b+1} &= 0,
    \end{align*}
    where $a_i^j = (U_d x_i^1)_j^2 - (U_d x_i^2)_j^2$. Since $2b + 1 > b$, the linear system above has infinite non-zero solutions, which finishes the proof of the second part.
\end{proof}

\subsection{Proof of Theorem \ref{thm:bijection}}

\begin{restatetheorem}{\ref{thm:bijection}}
    Let $f$ be a bijection mapping from $[1,b]$ to $[1,b]$, then $\kl{\bar{P}}{\bar{Q}} \le \frac{1}{b} \sum_{i=1}^b \kl{P_i}{Q_{f(i)}}$, and $\W(\bar{P}, \bar{Q}) \le \frac{1}{b} \sum_{i=1}^b \W(P_i, Q_{f(i)})$, where $P_i$ is the probability measure defined by $p_i$ (respectively for $Q_i$).
\end{restatetheorem}
\begin{proof}
    Recall that $\bar{p}(x) = \frac{1}{b} \sum_{i=1}^b p_i(x)$ and $\bar{q}(x) = \frac{1}{b} \sum_{i=1}^b q_i(x)$, we then have
    \begin{align*}
        \kl{\bar{P}}{\bar{Q}} &= \int_\mathcal{X} \bar{p}(x) \log\prn*{\frac{\bar{p}(x)}{\bar{q}(x)}} \dif x \\
        &= - \int_\mathcal{X} \bar{p}(x) \log\prn*{\frac{1}{b} \sum_{i=1}^b \frac{p_i(x)}{\bar{p}(x)} \cdot \frac{q_{f(i)}(x)}{p_i(x)}} \dif x \\
        &\le - \int_\mathcal{X} \bar{p}(x) \frac{1}{b} \sum_{i=1}^b \frac{p_i(x)}{\bar{p}(x)} \log\prn*{\frac{q_{f(i)}(x)}{p_i(x)}} \dif x \\
        &= -\frac{1}{b} \sum_{i=1}^b \int_\mathcal{X} p_i(x) \log\prn*{\frac{q_{f(i)}(x)}{p_i(x)}} \dif x \\
        &= \frac{1}{b} \sum_{i=1}^b \kl{P_i}{Q_{f(i)}},
    \end{align*}
    where the only inequality follows by applying Jensen's inequality on the concave logarithmic function. This finishes the proof of the upper bound for KL divergence.

    To prove the counterpart for Wasserstein distance, we apply Lemma \ref{lm:kr_dual} on $\bar{P}$ and $\bar{Q}$:
    \begin{align*}
        \W(\bar{P}, \bar{Q}) &= \sup_{f \in Lip_1} \brc*{\int_\mathcal{X} f \dif \bar{P} - \int_\mathcal{X} f \dif \bar{Q}} \\
        &= \sup_{f \in Lip_1} \brc*{\int_\mathcal{X} f \dif \prn*{\frac{1}{b} \sum_{i=1}^b P_i} - \int_\mathcal{X} f \dif \prn*{\frac{1}{b} \sum_{i=1}^b Q_{f(i)}}} \\
        &\le \frac{1}{b} \sum_{i=1}^b \sup_{f \in Lip_1} \brc*{\int_\mathcal{X} f \dif P_i - \int_\mathcal{X} f \dif Q_{f(i)}} \\
        &= \frac{1}{b} \sum_{i=1}^b \W(P_i, Q_{f(i)}).
    \end{align*}
    The proof is complete.
\end{proof}

\subsection{Proof of Theorem \ref{thm:minimizer}}

\begin{restatetheorem}{\ref{thm:minimizer}}
    Suppose that $\{x_i^1\}_{i=1}^b$, $\{x_i^2\}_{i=1}^b$ are sorted in the same order, then $f(j) = j$ is the minimizer of $\sum_{i=1}^b \kl{P_i}{Q_{f(i)}}$ and $\sum_{i=1}^b \W(P_i, Q_{f(i)})$.
\end{restatetheorem}
\begin{proof}
    For simplicity, we assume that all data points of $\{x_i^1\}_{i=1}^b$ and $\{x_i^2\}_{i=1}^b$ are different from each other. Since $P_i$ and $Q_i$ are Gaussian distributions with the same variance, the KL divergence and Wasserstein distance between them could be analytically acquired:
    \begin{equation*}
        \kl{P_i}{Q_j} = \frac{(x_i^1 - x_j^2)^2}{2\sigma^2}, \quad \textrm{and} \quad \W(P_i, Q_j) = \abs{x_i^1 - x_j^2}.
    \end{equation*}
    Suppose there exists $i \in [1,b]$ such that $f(i) \ne i$. Without loss of generality, we assume that $f(i) > i$. Then by the pigeonhole principle, there exists $j \in (i,b]$ that satisfies $f(j) < f(i)$. Suppose that $\{x_i^1\}_{i=1}^b$, $\{x_i^2\}_{i=1}^b$ are both sorted in ascending order, we have $x_i^1 < x_j^1$ and $x_{f(i)}^2 > x_{f(j)}^2$. For any $p \in \{1, 2\}$, the following $3$ cases cover all possible equivalent combinations of the order of $x_i^1$, $x_j^1$, $x_{f(j)}^2$ and $x_{f(i)}^2$:
    \begin{itemize}
        \item When $x_i^1 < x_j^1 < x_{f(j)}^2 < x_{f(i)}^2$ and $p = 2$, we have
        \begin{align*}
            &(x_i^1 - x_{f(i)}^2)^2 + (x_j^1 - x_{f(j)}^2)^2 - (x_i^1 - x_{f(j)}^2)^2 - (x_j^1 - x_{f(i)}^2)^2 \\
            &= (2x_i^1 - x_{f(i)}^2 - x_{f(j)}^2)(x_{f(j)}^2 - x_{f(i)}^2) - (2x_j^1 - x_{f(j)}^2 - x_{f(i)}^2)(x_{f(j)}^2 - x_{f(i)}^2) \\
            &= (x_{f(j)}^2 - x_{f(i)}^2)(2x_i^1 - 2x_j^1) > 0.
        \end{align*}
        Elsewise when $p = 1$, we have
        \begin{equation*}
            \abs{x_i^1 - x_{f(i)}^2} + \abs{x_j^1 - x_{f(j)}^2} = \abs{x_i^1 - x_{f(j)}^2} + \abs{x_j^1 - x_{f(i)}^2}.
        \end{equation*}
        \item When $x_i^1 < x_{f(j)}^2 < x_j^1 < x_{f(i)}^2$, we have
        \begin{equation*}
            \abs{x_i^1 - x_{f(i)}^2}^p > \abs{x_i^1 - x_{f(j)}^2}^p + \abs{x_j^1 - x_{f(i)}^2}^p.
        \end{equation*}
        \item When $x_i^1 < x_{f(j)}^2 < x_{f(i)}^2 < x_j^1$, we have
        \begin{align*}
            \abs{x_i^1 - x_{f(i)}^2}^p + \abs{x_j^1 - x_{f(j)}^2}^p &\ge \abs{x_i^1 - x_{f(j)}^2}^p + \abs{x_{f(i)}^2 - x_{f(j)}^2}^p \\
            &\quad+ \abs{x_j^1 - x_{f(i)}^2}^p + \abs{x_{f(i)}^2 - x_{f(j)}^2}^p \\
            &> \abs{x_i^1 - x_{f(j)}^2}^p + \abs{x_j^1 - x_{f(i)}^2}^p.
        \end{align*}
    \end{itemize}
    In conclusion, under all possible circumstances, we have $\abs{x_i^1 - x_{f(i)}^2}^p + \abs{x_j^1 - x_{f(j)}^2}^p \ge \abs{x_i^1 - x_{f(j)}^2}^p + \abs{x_j^1 - x_{f(i)}^2}^p$, which implies that by setting $f'(i) = f(j)$, $f'(j) = f(i)$ and $f'(k) = f(k)$ for $k \notin \{i, j\}$, $f'$ will be a better choice over $f$ to minimize $\kl{\bar{P}}{\bar{Q}}$ or $\W(\bar{P}, \bar{Q})$. The proof is complete since the existence of a minimizer is obvious.
\end{proof}

\section*{Further Discussions} \label{sec:discuss}

\subsection{Information Bottleneck for Target-domain Generalization}
Alternatively, one can also decompose the representation space distribution shift from an anti-causal perspective:
\begin{equation*}
    I(R,Y;D) \textrm{ (distribution shift)} = I(Y;D) \textrm{ (label shift)} + I(R;D|Y) \textrm{ (concept shift)}.
\end{equation*}
While label shift $I(Y;D)$ is an intrinsic property of the data distribution and cannot be optimized by learning algorithms, the anti-causal concept shift $I(R;D|Y)$ is closely connected to the IB principle: Notice the Markov chain $(D,Y) - X - R$, we then have that by applying the data-processing inequality,
\begin{equation*}
    I(R;D|Y) = I(R;D,Y) - I(R;Y) \le I(R;X) - I(R;Y).
\end{equation*}
Recall that the spirit of IB is to minimize $I(X;R)$ while maximizing $I(R;Y)$, this target can be achieved by solely penalizing $I(R;X|Y)$ \cite{fischer2020conditional, federici2019learning, lee2021compressive}. Therefore, when there is no label imbalance issues ($I(Y;D) = 0$), the anti-causal concept shift $I(R;D|Y)$ can be minimized by the IB principle:
\begin{mdframed}
    Information bottleneck promotes test-domain generalization if $I(Y;D) \rightarrow 0$.
\end{mdframed}
Notably, our analysis facilitates previous works \cite{ibirm2021} utilizing IB to enhance the performance of OOD generalization. While the analysis of \cite{ibirm2021} is primarily restricted to linear models, our results apply to any encoder-classifier type network. Similar to Proposition \ref{prop:cov_shift_emp}, we provide the following result on the feasibility of utilizing the empirical IB $I(R_i;X_i|Y_i)$ as a proxy to optimize $I(R;X|Y)$.
\begin{proposition} \label{prop:ib_emp}
    Under mild conditions, we have
    \begin{equation*}
        \skl{P_{R,X,Y}}{P_{R_i,X_i,Y_i}} = O\prn*{\sqrt{I(W;D_i)}}.
    \end{equation*}
\end{proposition}
\begin{proof}
    The proof follows the same development as Proposition \ref{prop:cov_shift_emp}.
\end{proof}

\subsection{Leveraging the Independence Assumption}
In the main text, we only assume that the test domains are independent of the training domains, while the training domains are not necessarily independent of each other (same for test domains). This assumption is much weaker than the i.i.d domains assumption adopted in \cite{qrm2022} by allowing correlations between training domains, e.g. sampling from a finite set without replacement. While this weaker assumption is preferable, we highlight that the test-domain generalization bounds in Theorem \ref{thm:pp_con} and \ref{thm:pp_best} can be further tightened by a factor of $1/m'$ when the i.i.d condition of test domains is incorporated. Therefore, one can now guarantee better generalization by increasing the number of domains, which is consistent with real-world observations.

\begin{theorem}
    If Assumption \ref{asmp:loss_subgauss} holds and the test domains $D_t$ are independent of each other, then
    \begin{equation*}
        \Pr\brc*{\abs{L_t(w) - L(w)} \ge \epsilon} \le \frac{2\sigma^2}{m'\epsilon^2} I(Z;D).
    \end{equation*}
\end{theorem}
\begin{proof}
    If Assumption \ref{asmp:loss_subgauss} holds, then for any $d \in \mathcal{D}$, by setting $P = P_Z$, $Q = P_{Z|D=d}$ and $f(z) = \ell(f_w(x),y)$ in Lemma \ref{lm:kl_con}, we have
    \begin{align*}
        (L_d(w) - L(w))^2 &= \prn*{\E_{Z|D=d}[\ell(f_w(X),Y)] - \E_Z[\ell(f_w(X),Y]}^2 \\
        &\le \prn*{\sqrt{2\sigma^2 \kl{P_{Z|D=d}}{P_Z}}}^2 \\
        &\le 2\sigma^2 \kl{P_{Z|D=d}}{P_Z}.
    \end{align*}
    Taking the expectation over $D \sim \nu$, we can get
    \begin{align*}
        \E_D[(L_D(w) - L(w))^2] &\le \E_D\brk*{2\sigma^2 \kl{P_{Z|D=d}}{P_Z}} \\
        &= 2\sigma^2 \kl{P_{Z|D}}{P_Z} \\
        &= 2\sigma^2 I(Z;D).
    \end{align*}
    When the test domains $\{D_k\}_{i=k}^{m'}$ are independent of each other, they can be regarded as i.i.d copies of $D$, i.e.
    \begin{align*}
        \Var_{D_t}[L_t(w)] &= \frac{1}{m'^2} \sum_{k=1}^{m'} \Var_{D_k}[L_{D_k}(w)] \\
        &= \frac{1}{m'^2} \sum_{k=1}^{m'} \E_{D_k}[(L_{D_k}(w) - L(w))^2] \\
        &\le \frac{1}{m'^2} \sum_{k=1}^{m'} 2\sigma^2 I(Z;D) \\
        &= \frac{1}{m'} 2\sigma^2 I(Z;D).
    \end{align*}
    Finally, by applying Chebyshev's inequality, we can prove that
    \begin{equation*}
        \Pr\brc*{\abs{L_t(w) - L(w)} \ge \epsilon} \le \frac{2\sigma^2}{m'\epsilon^2} I(Z;D).
    \end{equation*}
\end{proof}

\begin{theorem}
    If Assumption \ref{asmp:loss_bounded} and \ref{asmp:loss_dist} hold and the test domains $D_t$ are independent of each other, then for any $w = (\phi, \psi) \in \mathcal{W}$,
    \begin{equation*}
        \Pr\brc*{L_t(\psi) - L(\psi) \ge \epsilon + L^*} \le \frac{2\sigma^2}{m'\epsilon^2} I(X;D),
    \end{equation*}
    where $L^* = \min_{f^*} L_t(f^*) + L(f^*)$.
\end{theorem}
\begin{proof}
    For any environment $d \in \mathcal{D}$, classifier $\psi$ and $f^*: \mathcal{R} \mapsto \mathcal{Y}$, denote
    \begin{gather*}
        L_d(\psi, f^*) = \E_{R|D=d}[\ell(f_\psi(R), f^*(R))], \\
        L(\psi, f^*) = \E_R[\ell(f_\psi(R), f^*(R))].
    \end{gather*}
    By setting $P = P_R$, $Q = P_{R|D=d}$ and $f(R) = \ell(f_\psi(R), f^*(R))$ and applying Lemma \ref{lm:kl_con}, we have
    \begin{equation*}
        \prn*{L_d(\psi, f^*) - L(\psi, f^*)}^2 \le 2\sigma^2 \kl{P_{X|D=d}}{P_X}.
    \end{equation*}
    By taking the expectation over $D \sim \nu$, we get
    \begin{align*}
        \E_D \brk*{\prn*{L_d(\psi, f^*) - L(\psi, f^*)}^2} &\le 2\sigma^2 \kl{P_{X|D}}{P_X} \\
        &= 2\sigma^2 I(X;D).
    \end{align*}
    Through a similar procedure of proving Theorem \ref{thm:pp_con}, we have
    \begin{equation*}
        \Pr\brc*{\abs{L_d(\psi, f^*) - L(\psi, f^*)} \ge \epsilon} \le \frac{2\sigma^2}{m'\epsilon^2} I(X;D).
    \end{equation*}
    Recall that in Theorem \ref{thm:pp_best} we proved
    \begin{gather*}
        L_d(\psi, f^*) \le L_d(\psi) + L_d(f^*). \\
        L(\psi, f^*) \le L(\psi) + L(f^*). \\
        L_d(\psi, f^*) \ge L_d(\psi) - L_d(f^*). \\
        L(\psi, f^*) \ge L(\psi) - L(f^*).
    \end{gather*}
    Combining the results above, we have that with probability at least $1 -\frac{2\sigma^2}{m'\epsilon^2} I(X;D)$,
    \begin{align*}
        L_t(\psi) &\le L_t(\psi, f^*) + L_t(f^*) \\
        &\le L(\psi, f^*) + \epsilon + L_t(f^*) \\
        &\le L(\psi) + L(f^*) + \epsilon + L_t(f^*).
    \end{align*}
    By minimizing $L_t(f^*) + L(f^*)$, it follows that
    \begin{equation*}
        \Pr\brc*{L_t(\psi) - L(\psi) \ge \epsilon + \min_{f^*} (L_t(f^*) + L(f^*))} \le \frac{2\sigma^2}{m'\epsilon^2} I(X;D),
    \end{equation*}
    which finishes the proof.
\end{proof}

Furthermore, when the training domains satisfy the i.i.d condition, we prove in Theorem \ref{thm:ep_con} that $\sum_{i=1}^m I(W;D_i) \le I(W;D_s)$. Otherwise, we can only guarantee that for any $i \in [1,m]$, $I(W;D_i) \le I(W;D_s)$. This indicates that while the model achieves source-domain generalization by letting $I(W;D_i) \rightarrow 0$, it still learns from the training domains $D_s$. Notably, having $I(W;D_i) = 0$ for each $i \in [1,m]$ does not necessarily lead to $I(W;D_s) = 0$. To see this, we take $D_i$ as i.i.d random binary variables such that $\Pr(D_i = 0) = \Pr(D_i = 1) = \frac{1}{2}$, and let $W = D_1 \oplus \cdots \oplus D_m$, where $\oplus$ is the XOR operator. Then it is easy to verify that $W$ is independent of each $D_i$ since $P_{W|D_i} = P_W$, implying $I(W;D_i) = 0$. However, $I(W;D_s) = H(W)$ is strictly positive.

\subsection{High-probability Problem Formulation}
Another high-probability formulation of the DG problem is presented by \cite{qrm2022}, namely Quantile Risk Minimization (QRM). Under our notations, the QRM objective can be expressed as:
\begin{equation*}
    \min_w \epsilon \quad s.t. \quad \Pr\{L_t(w) \ge \epsilon\} \le \delta.
\end{equation*}
The main difference between our formulation in Problem \ref{prb:dg_hp} and QRM is that we not only consider the randomness of $D_t$, but also those of $D_s$ and $W$. This randomized nature of training domains and the hypothesis serve as the foundation of our information-theoretic generalization analysis. When the training-domain risks have been observed, i.e. $L_s(W)$ is fixed, our formulation reduces to QRM. The main advantage of our formulation is that it could be directly optimized by learning algorithms without further assumptions or simplifications. On the contrary, \cite{qrm2022} needs to further adopt kernel density estimation to approximate the quantile of the risks and transform the QRM problem to the empirical one (EQRM). Further advantages of our formulation include:
\begin{itemize}
    \item It is questionable to use the training risk distribution as a substitute for test risk distribution to minimize the quantile objective. Since the hypothesis and the training domains are dependent, there exists a gap between these two distributions which is not addressed by \cite{qrm2022}.
    \item Kernel density estimation required by QRM is challenging when the number of training domains is not sufficiently large. For comparison, IDM could be easily applied as long as there are at least $2$ training domains.
    \item Our formulation aims to find the optimal learning algorithm instead of the optimal hypothesis. This would be essential to analyze the correlations between the hypothesis $W$ and training domains $D_s$, and also is more suitable in robust learning settings when measuring the error bar.
    \item Our formulation directly characterizes the trade-off between optimization and generalization, which is the main challenge to achieve invariance across different domains \cite{irm2019}.
\end{itemize}

\subsection{Tighter Bounds for Target-Domain Population Risk}
In a similar vein, we provide the following two upper bounds for test-domain generalization error in terms of Wasserstein distances. For the following analysis, we assume the independence between training and test domains.
\begin{theorem} \label{thm:pp_wass_con}
    If Assumption \ref{asmp:loss_lip} holds, then for any $w \in \mathcal{W}$,
    \begin{equation*}
        \Pr\brc*{\abs{L_t(w) - L(w)} \ge \epsilon} \le \frac{\beta^2}{m'\epsilon^2} \E_D[\W^2(P_{Z|D=d}, P_Z)].
    \end{equation*}
    Furthermore, if the metric $d$ is discrete, then
    \begin{equation*}
        \Pr\brc*{\abs{L_t(w) - L(w)} \ge \epsilon} \le \frac{\beta^2}{2m'\epsilon^2} I(Z;D).
    \end{equation*}
\end{theorem}
\begin{proof}
    For any $d \in \mathcal{D}$, by setting $P = P_{Z|D=d}$, $Q = P_Z$ and $f(z) = \frac{1}{\beta} \ell(f_w(x),y)$ in Lemma \ref{lm:kr_dual}, we have
    \begin{equation*}
        (L_d(w) - L(w))^2 \le \beta^2 \W^2(P_{Z|D=d}, P_Z).
    \end{equation*}
    Following a similar procedure with the proof of Theorem \ref{thm:pp_con}, we have
    \begin{equation*}
        \Pr\brc*{\abs{L_t(w) - L(w)} \ge \epsilon} \le \frac{\beta^2}{m'\epsilon^2} \E_D[\W^2(P_{Z|D=d}, P_Z)].
    \end{equation*}
    When the metric $d$ is discrete, Wasserstein distance is equivalent to the total variation. Therefore
    \begin{align*}
        \Pr\brc*{\abs{L_t(w) - L(w)} \ge \epsilon} &\le \frac{\beta^2}{m'\epsilon^2} \E_D[\TV^2(P_{Z|D=d}, P_Z)] \\
        &\le \frac{\beta^2}{m'\epsilon^2} \E_D\brk*{\frac{1}{2} \kl{P_{Z|D=d}}{P_Z}} \\
        &= \frac{\beta^2}{2m'\epsilon^2} I(Z;D),
    \end{align*}
    where the second inequality is by applying Lemma \ref{lm:pinsker}. The proof is complete.
\end{proof}
\begin{theorem} \label{thm:pp_wass_best}
    If Assumption \ref{asmp:loss_dist} holds, and $\ell(f_w(X), f_{w'}(X))$ is $\beta$-Lipschitz for any $w, w' \in \mathcal{W}$, then for any $w \in \mathcal{W}$
    \begin{equation*}
        \Pr\brc*{L_t(w) - L(w) \ge \epsilon + L^*} \le \frac{\beta^2}{m'\epsilon^2} \E_D[\W^2(P_{X|D=d},P_{X})],
    \end{equation*}
    where $L^* = \min_{w^* \in \mathcal{W}}(L_t(w^*) + L(w^*))$. Furthermore, if the metric $d$ is discrete, then
    \begin{equation*}
        \Pr\brc*{L_t(w) - L(w) \ge \epsilon + L^*} \le \frac{\beta^2}{2m'\epsilon^2} I(X;D).
    \end{equation*}
\end{theorem}
\begin{proof}
    Following the proof sketch of Theorem \ref{thm:pp_best}, by setting $P = P_X$, $Q = P_{X|D=d}$ and $f(x) = \ell(f_w(x), f_{w'}(x))$ in Lemma \ref{lm:kr_dual}, we have
    \begin{equation*}
        \prn*{L_d(w,w') - L(w,w')}^2 \le \beta^2 \W^2(P_{X|D=d}, P_X),
    \end{equation*}
    for any $d \in \mathcal{D}$ and $w, w' \in \mathcal{W}$. Similarly, by applying the independence of $\{D_k\}_{k=1}^{m'}$ and Chebyshev's inequality, we have
    \begin{equation*}
        \Pr\brc*{\abs{L_t(w,w') - L(w,w')} \ge \epsilon} \le \frac{\beta^2}{m'\epsilon^2} \E_D[\W^2(P_{X|D=d}, P_X)].
    \end{equation*}
    Through a similar procedure of proving Theorem \ref{thm:pp_best} and Theorem \ref{thm:pp_wass_con}, we can get
    \begin{align*}
        \Pr\brc*{L_t(w) - L(w) \ge \epsilon + \min_{w^* \in \mathcal{W}} (L_t(w^*) + L(w^*))} &\le \frac{\beta^2}{m'\epsilon^2} \E_D[\W^2(P_{X|D=d}, P_X)] \\
        &\le \frac{\beta^2}{2m'\epsilon^2} I(X;D),
    \end{align*}
    which finishes the proof.
\end{proof}
The expected Wasserstein distances can be regarded as analogs to the mutual information terms $I(Z;D)$ and $I(X;D)$ respectively, through a similar reduction procedure as depicted in (\ref{eq:wass_reduce}). This also provides alternative perspectives to the distribution shift and the covariate shift in (\ref{eq:dist_shift}).

Moreover, these bounds can be further tightened by considering the shift in the risks. Given a hypothesis $w \in \mathcal{W}$ and domain $d \in \mathcal{D}$, let $L = \ell(f_w(X),Y)$ be the risk of predicting some random sample $Z \sim P_{Z|D=d}$. Then by applying Lemma \ref{lm:kl_con} with $P = P_L$, $Q = P_{L|D=d}$, $f(x) = x$ and Assumption \ref{asmp:loss_bounded}, we have
\begin{align*}
    \E_D[(L_D(w) - L(w))^2] &= \E_D[(\E_{L|D=d}[L] - \E_L[L])^2] \\
    &\le \E_D\brk*{\prn*{\sqrt{\frac{M^2}{2} \kl{P_{L|D=d}}{P_L}}}^2} \\
    &= \frac{M^2}{2} I(L;D).
\end{align*}
Through the similar sketch of proving Theorem \ref{thm:pp_con}, we can prove that
\begin{equation*}
    \Pr\{\abs{L_t(w) - L(w)} \ge \epsilon\} \le \frac{M^2}{2\epsilon^2} I(L;D).
\end{equation*}
By the Markov chain relationship $D \rightarrow (X,Y) \rightarrow (f_\phi(X),Y) \rightarrow (f_w(X),Y) \rightarrow L$, this upper bound is strictly tighter than Theorem \ref{thm:pp_con} which uses sample space $I(Z;D)$ or representation space $I(R,Y;D)$ distribution shifts. Also, notice that the mutual information $I(L;D)$ could be rewritten as $\kl{P_{L|D}}{P_L}$, this suggests that matching the inter-domain distributions of the risks helps to generalize on test domains. This observation facilitates the work of \cite{vrex2021}, which proposes to align the empirical risks of distinct training domains. Still, the gap between the risk shift $I(L;D)$ and its empirical counterpart $I(L_i;D_i)$ is not addressed by \cite{vrex2021}, which further requires minimizing $I(W;D_i)$. Considering that $L$ is a scalar while $R$ is a vector, aligning the distributions of the risks avoids high-dimensional distribution matching, and thus enables more efficient implementation. We will leave this for future research.

\subsection{Generalization Bounds for Source-Domain Empirical Risk}
The information-theoretic technique adopted in this paper to derive generalization bounds is closely related to the recent advancements of information-theoretic generalization analysis. Specifically, Theorem \ref{thm:pp_domain} can be viewed as a multi-domain version of the standard generalization error bound in supervised learning \cite{xu2017information, bu2020tightening}. In this section, we further consider the effect of finite training samples, which further raises a gap between domain-level population and empirical risks. In addition to the generalization bounds for population risks established in Section \ref{sec:theory}, upper bounds for the empirical risk can also be derived by incorporating Assumption \ref{asmp:loss_bounded}:
\begin{theorem} \label{thm:ep_con}
    Let $W$ be the output of learning algorithm $\mathcal{A}$ with input $S$. If Assumption \ref{asmp:loss_bounded} holds, then
    \begin{align*}
        \abs{\E_{W,D_s,S}[L'(W)] - \E_W[L(W)]} &\le \frac{1}{m} \sum_{i=1}^m \sqrt{\frac{M^2}{2}I(W;D_i)} \\
        &\quad+ \frac{1}{mn} \sum_{i=1}^m \sum_{j=1}^n \sqrt{\frac{M^2}{2}I(W;Z_j^i|D_i)}, \\
        \Pr\brc*{\abs{\E_{W,D_s,S}[L'(W)] - \E_W[L(W)]} \ge \epsilon} &\le \frac{M^2}{mn\epsilon^2} \prn*{I(W;S) + \log 3}.
    \end{align*}
\end{theorem}
\begin{proof}
    Recall that any random variables bounded by $[0, M]$ are $\frac{M}{2}$-subgaussian. From assumption \ref{asmp:loss_bounded}, we know that $\ell(f_W(X),Y)$ is $\frac{M}{2}$-subgaussian w.r.t $P_W \circ P_Z$. Then by applying Lemma \ref{lm:kl_con}, we have
    \begin{align*}
        &\abs{\E_{W,D_s,S}[L'(W)] - \E_W[L(W)]} \\
        &= \abs*{\frac{1}{m} \sum_{i=1}^m \frac{1}{n} \sum_{j=1}^n \E_{W,D_i,Z_j^i} [\ell(f_W(X_j^i),Y_j^i)] - \E_{W,D,Z}[\ell(f_W(X),Y)]} \\
        &\le \frac{1}{mn} \sum_{i=1}^m \sum_{j=1}^n \abs*{\E_{W,D_i,Z_j^i} [\ell(f_W(X_j^i),Y_j^i)] - \E_{W,D,Z}[\ell(f_W(X),Y)]} \\
        &\le \frac{1}{mn} \sum_{i=1}^m \sum_{j=1}^n \sqrt{\frac{M^2}{2} \kl*{P_{W,D_i,Z_j^i}}{P_W P_{D_i,Z_j^i}}}.
    \end{align*}
    Notice that for any $D \in D_s$ and $Z \in S_D$,
    \begin{align*}
        \kl*{P_{W,D,Z}}{P_W P_{D,Z}} &= \E_{W,D,Z} \brk*{\log \frac{P_{W,D,Z}}{P_W P_{D,Z}}} \\
        &= I(W;D,Z) \\
        &= I(W;D) + I(W;Z|D).
    \end{align*}
    Combining our results above, we then get
    \begin{align*}
        \abs{\E_{W,D_s,S}[L'(W)] - \E_W[L(W)]} &\le \frac{1}{mn} \sum_{i=1}^m \sum_{j=1}^n \sqrt{\frac{M^2}{2} (I(W;D_i) + I(W;Z_j^i|D_i))} \\
        &\le \frac{1}{mn} \sum_{i=1}^m \sum_{j=1}^n \prn*{\sqrt{\frac{M^2}{2}I(W;D_i)} + \sqrt{\frac{M^2}{2}I(W;Z_j^i|D_i)}} \\
        &= \frac{1}{m} \sum_{i=1}^m \sqrt{\frac{M^2}{2}I(W;D_i)} + \frac{1}{mn} \sum_{i=1}^m \sum_{j=1}^n \sqrt{\frac{M^2}{2}I(W;Z_j^i|D_i)}.
    \end{align*}
    Similarly, we have
    \begin{align*}
        &\E_{W,S}\brk*{(L'(W) - \E_W[L(W)])^2} \\
        &= \E_{W,S}\brk*{\prn*{\frac{1}{m} \sum_{i=1}^m \frac{1}{n} \sum_{j=1}^n \ell(f_W(X_j^i),Y_j^i) - \E_W[L(W)]}^2} \\
        &= \frac{M^2}{mn} \prn*{I(W;S) + \log 3}.
    \end{align*}
    This further implies by Lemma \ref{lm:kl_con_pair} that
    \begin{equation*}
        \Pr\brc*{\abs{\E_{W,D_s,S}[L'(W)] - \E_W[L(W)]} \ge \epsilon} \le \frac{M^2}{mn\epsilon^2} \prn*{I(W;S) + \log 3},
    \end{equation*}
    which completes the proof.
\end{proof}
\begin{theorem} \label{thm:ep_wass}
    Let $W$ be the output of learning algorithm $\mathcal{A}$ under training domains $D_s$. If $\ell(f_w(X),Y)$ is $\beta'$-Lipschitz w.r.t $w$, then
    \begin{align*}
        \abs{\E_{W,D_s,S}[L'(W)] - \E_W[L(W)]} &\le \frac{\beta'}{m} \sum_{i=1}^m \E_{D_i} [\W(P_{W|D_i},P_W)] \\
        &\quad+ \frac{\beta'}{mn} \sum_{i=1}^m \sum_{j=1}^n \E_{D_i,Z_j^i} [\W(P_{W|D_i,Z_j^i},P_{W|D_i})].
    \end{align*}
\end{theorem}
\begin{proof}
    Recall the proof of Theorem \ref{thm:ep_con}, we have
    \begin{align*}
        &\abs{\E_{W,D_s,S}[L'(W)] - \E_W[L(W)]} \\
        &\le \frac{1}{mn} \sum_{i=1}^m \sum_{j=1}^n \abs*{\E_{W,D_i,Z_j^i} [\ell(f_W(X_j^i),Y_j^i)] - \E_{W,D,Z}[\ell(f_W(X),Y)]} \\
        &\le \frac{1}{mn} \sum_{i=1}^m \sum_{j=1}^n \abs*{\E_{W,D_i,Z_j^i} [\ell(f_W(X_j^i),Y_j^i)] - \E_{W,D_i,Z}[\ell(f_W(X),Y)]} \\
        &\qquad+ \frac{1}{mn} \sum_{i=1}^m \sum_{j=1}^n \abs*{\E_{W,D_i,Z}[\ell(f_W(X),Y)] - \E_{W,D,Z}[\ell(f_W(X),Y)]} \\
        &\le \frac{1}{mn} \sum_{i=1}^m \sum_{j=1}^n \E_{D_i,Z_j^i} \abs*{\E_{W|D_i,Z_j^i} [\ell(f_W(X_j^i),Y_j^i)] - \E_{W|D_i}[\ell(f_W(X_j^i),Y_j^i)]} \\
        &\qquad+ \frac{1}{mn} \sum_{i=1}^m \sum_{j=1}^n \E_{D_i} \abs*{\E_{W|D_i}[L_{D_i}(W)] - \E_W[L_{D_i}(W)]} \\
        &\le \frac{\beta'}{mn} \sum_{i=1}^m \sum_{j=1}^n \prn*{\E_{D_i,Z_j^i} [\W(P_{W|D_i,Z_j^i},P_{W|D_i})] + \E_{D_i} [\W(P_{W|D_i},P_W)]} \\
        &= \frac{\beta'}{m} \sum_{i=1}^m \E_{D_i} [\W(P_{W|D_i},P_W)] + \frac{\beta'}{mn} \sum_{i=1}^m \sum_{j=1}^n \E_{D_i,Z_j^i} [\W(P_{W|D_i,Z_j^i},P_{W|D_i})],
    \end{align*}
    where the last inequality is by applying Lemma \ref{lm:kr_dual}. The proof is complete.
\end{proof}
The theorems above provide upper bounds for the empirical generalization risk by exploiting the mutual information between the hypothesis and the samples (or the Wasserstein distance counterparts). Compared to Theorems \ref{thm:pp_domain} and \ref{thm:pp_domain_wass}, these results additionally consider the randomness of the sampled data points $Z_j^i$, and indicate that traditional techniques that improve the generalization of deep learning algorithms by minimizing $I(W;S)$, such as gradient clipping \cite{wang2021generalization, wang2022information} and stochastic gradient perturbation \cite{pensia2018generalization, wang2021analyzing} methods, also enhance the capability of learning algorithm $\mathcal{A}$ to generalize on target domains under our high-probability problem setting by preventing overfitting to training samples. This observation is also verified in \cite{wang2022information}. Relevant analysis may also motivate information-theoretic generalization analysis for meta-learning tasks \cite{chen2021generalization, jose2021information, hellstrom2022evaluated, bu2023generalization}. We do not consider these approaches in this paper, as they are beyond the scope of solving Problem \ref{prb:dg_hp}.

\section*{Experiment Details} \label{sec:setting}
In this paper, deep learning models are trained with an Intel Xeon CPU (2.10GHz, 48 cores), 256GB memory, and 4 Nvidia Tesla V100 GPUs (32GB).

\subsection{Implementation of IDM}
We provide the pseudo-code for PDM in Algorithm \ref{alg:PDM}, where the moving averages $X_{ma}^i$ are initialized with $0$. The input data points for distribution alignment are represented as matrices $X^i \in \mathbb{R}^{b \times d}$, where $b$ denotes the batch size and $d$ represents the dimensionality. Each row of $X$ then corresponds to an individual data point. We also present the pseudo-code for IDM in Algorithm \ref{alg:IDM} for completeness.

\begin{algorithm}[ht]
\caption{PDM for distribution matching.}
\label{alg:PDM}
\begin{algorithmic}[1]
\STATE \textbf{Input:} Data matrices $\{X^i\}_{i=1}^m$, moving average $\gamma$.
\STATE \textbf{Output:} Penalty of distribution matching.
\FOR{$i$ \textbf{from} $1$ \textbf{to} $m$}
\STATE Sort the elements of $X^i$ in each column in ascending order.
\STATE Calculate moving average $X_{ma}^i = \gamma X_{ma}^i + (1-\gamma) X^i$.
\ENDFOR
\STATE Calculate the mean of data points across domains: $X_{ma} = \frac{1}{m} \sum_{i=1}^m X_{ma}^i$.
\STATE \textbf{Output:} $\mathcal{L}_{\mathrm{PDM}} = \frac{1}{mdb} \sum_{i=1}^m \norm{X_{ma} - X_{ma}^i}_F^2$.
\end{algorithmic}
\end{algorithm}

\begin{algorithm}[ht]
\caption{IDM for high-probability DG.}
\label{alg:IDM}
\begin{algorithmic}[1]
\STATE \textbf{Input:} Model $W$, training dataset $S$, hyper-parameters $\lambda_1$, $\lambda_2$, $t_1$, $t_2$, $\gamma_1$, $\gamma_2$.
\FOR{$t$ \textbf{from} $1$ \textbf{to} \#steps}
\FOR{$i$ \textbf{from} $1$ \textbf{to} $m$}
\STATE Randomly sample a batch $B_t^i = (X_t^i, Y_t^i)$ from $S_{D_i}$ of size $b$.
\STATE Compute individual representations: $(R_t^i)_j = f_\Phi\prn*{(X_t^i)_j}$, for $j \in [1,b]$.
\STATE Compute individual risks: $(L_t^i)_j = \ell\prn*{f_\Psi\prn*{(R_t^i)_j},(Y_t^i)_j}$, for $j \in [1,b]$.
\STATE Compute individual gradients: $(G_t^i)_j = \nabla_\Psi (L_t^i)_j$, for $j \in [1,b]$.
\ENDFOR
\STATE Compute total empirical risk: $\mathcal{L}_{\mathrm{IDM}} = \frac{1}{mn} \sum_{i=1}^m \sum_{j=1}^n (L_t^i)_j$.
\IF{$t \ge t_1$}
\STATE Compute gradient alignment risk: $\mathcal{L}_{\mathrm{G}} = \mathrm{PDM}(\{G_t^i\}_{i=1}^m, \gamma_1)$.
\STATE $\mathcal{L}_{\mathrm{IDM}} = \mathcal{L}_{\mathrm{IDM}} + \lambda_1 \mathcal{L}_{\mathrm{G}}$.
\ENDIF
\IF{$t \ge t_2$}
\STATE Compute representation alignment risk: $\mathcal{L}_{\mathrm{R}} = \mathrm{PDM}(\{R_t^i\}_{i=1}^m, \gamma_2)$.
\STATE $\mathcal{L}_{\mathrm{IDM}} = \mathcal{L}_{\mathrm{IDM}} + \lambda_2 \mathcal{L}_{\mathrm{R}}$.
\ENDIF
\STATE Back-propagate gradients $\nabla_W \mathcal{L}_{\mathrm{IDM}}$ and update the model $W$.
\ENDFOR
\end{algorithmic}
\end{algorithm}

We follow the experiment settings of \cite{fishr2022} and utilize a moving average to increase the equivalent number of data points for more accurate probability density estimation in distribution alignments. This does not invalidate our analysis in Theorem \ref{thm:indist}, as the maximum equivalent batch size ($b / (1 - \gamma) \approx 32 / (1 - 0.95) = 640$) remains significantly smaller than the dimensionality of the representation ($2048$ for ResNet-50 in DomainBed) or the gradient ($2048 \times c$, the number of classes) and satisfies $d > 2b + 1$. Therefore, it is still impossible to distinguish different inter-domain distributions as indicated by Theorem \ref{thm:indist}. However, this moving average technique indeed helps to improve the empirical performance, as shown by our ablation studies.

\subsection{Colored MNIST}
The Colored MNIST dataset is a binary classification task introduced by IRM \cite{irm2019}. The main difference between Colored MNIST and the original MNIST dataset is the manually introduced strong correlation between the label and image colors. Colored MNIST is generated according to the following procedure:
\begin{itemize}
    \item Give each sample an initial label by whether the digit is greater than $4$ (i.e. label $0$ for $0$-$4$ digits and label $1$ for $5$-$9$ digits.
    \item Randomly flip the label with probability $0.25$, so an oracle predictor that fully relies on the shape of the digits would achieve a $75\%$ accuracy.
    \item Each environment is assigned a probability $P_e$, which characterizes the correlation between the label and the color: samples with label $0$ have $P_e$ chance to be red, and $1 - P_e$ chance to be green, while samples with label $1$ have $P_e$ chance to be green, and $1 - P_e$ chance to be red.
\end{itemize}
The original environment setting of \cite{irm2019} includes two training domains $D_s = \{P_1 = 90\%, P_2 = 80\%\}$, such that the predictive power of the color superiors that of the actual digits. This correlation is reversed in the test domain $D_t = \{P_3 = 10\%\}$, thus fooling algorithms without causality inference abilities to overfit the color features and generalize poorly on test environments.

The original implementation\footnote{\url{https://github.com/facebookresearch/InvariantRiskMinimization}} uses a $3$-layer MLP network with ReLU activation. The model is trained for $501$ epochs in a full gradient descent scheme, such that the batch size equals the number of training samples $25,000$. We follow the hyper-parameter selection strategy of \cite{irm2019} through a random search over $50$ independent trials, as reported in Table \ref{tbl:cmnist_param} along with the parameters selected for IDM. Considering that the covariate shift is not prominent according to the dataset construction procedure, we only apply gradient alignment without feature alignment in this experiment.

\begin{table}[ht]
\centering
\small
\caption{The hyper-parameters of Colored MNIST.}
\label{tbl:cmnist_param}
\adjustbox{max width=\textwidth}{%
\begin{tabular}{llc}
\toprule
Parameter & Random Distribution & Selected Value \\
\midrule
dimension of hidden layer & $2^{\mathrm{Uniform}(6,9)}$ & $433$ \\
weight decay & $10^{\mathrm{Uniform}(-2,-5)}$ & $0.00034$ \\
learning rate & $10^{\mathrm{Uniform}(-2.5,-3.5)}$ & $0.000449$ \\
warmup iterations & $\mathrm{Uniform}(50, 250)$ & $154$ \\
regularization strength & $10^{\mathrm{Uniform}(4,8)}$ & $2888595.180638$ \\
\bottomrule
\end{tabular}}
\end{table}

\subsection{DomainBed Benchmark}
DomainBed \cite{domainbed2020} is an extensive benchmark for both DA and DG algorithms, which involves various synthetic and real-world datasets mainly focusing on image classification:
\begin{itemize}
    \item Colored MNIST \cite{irm2019} is a variant of the MNIST dataset. As discussed previously, Colored MNIST includes $3$ domains $\{90\%, 80\%, 10\%\}$, $70,000$ samples of dimension $(2,28,28)$ and $2$ classes.
    \item Rotated MNIST \cite{rmnist2015} is a variant of the MNIST dataset with $7$ domains $\{0,15,30,45,60,75\}$ representing the rotation degrees, $70,000$ samples of dimension $(28,28)$ and $10$ classes.
    \item VLCS \cite{fang2013unbiased} includes $4$ domains $\{\mathrm{Caltech101},\mathrm{LabelMe},\mathrm{SUN09},\mathrm{VOC2007}\}$, $10,729$ samples of dimension $(3,224,224)$ and $5$ classes.
    \item PACS \cite{li2017deeper} includes $4$ domains $\{\mathrm{art},\mathrm{cartoons},\mathrm{photos},\mathrm{sketches}\}$, $9,991$ samples of dimension $(3,224,224)$ and $7$ classes.
    \item OfficeHome \cite{venkateswara2017deep} includes $4$ domains $\{\mathrm{art},\mathrm{clipart},\mathrm{product},\mathrm{real}\}$, $15,588$ samples of dimension $(3,224,224)$ and $65$ classes.
    \item TerraIncognita \cite{beery2018recognition} includes $4$ domains $\{\mathrm{L100},\mathrm{L38},\mathrm{L43},\mathrm{46}\}$ representing locations of photographs, $24,788$ samples of dimension $(3,224,224)$ and $10$ classes.
    \item DomainNet \cite{peng2019moment} includes $6$ domains $\{\mathrm{clipart},\mathrm{infograph},\mathrm{painting},\mathrm{quickdraw},\mathrm{real},\mathrm{sketch}\}$, $586,575$ samples of dimension $(3,224,224)$ and $345$ classes.
\end{itemize}

We list all competitive DG approaches below. Note that some recent progress is omitted \cite{cha2021swad, qrm2022, wang2022causal, wang2023pgrad, setlur2023bitrate, chen2023pareto}, which either contributes complementary approaches, does not report full DomainBed results, or does not report the test-domain validation scores. Due to the limitation of computational resources, we are not able to reproduce the full results of these works on DomainBed.
\begin{itemize}
    \item ERM: Empirical Risk Minimization.
    \item IRM: Invariant Risk Minimization \cite{irm2019}.
    \item GroupDRO: Group Distributionally Robust Optimization \cite{groupdro2019}.
    \item Mixup: Interdomain Mixup \cite{mixup2020}.
    \item MLDG: Meta Learning Domain Generalization \cite{mldg2018}.
    \item CORAL: Deep CORAL \cite{coral2016}.
    \item MMD: Maximum Mean Discrepancy \cite{mmd2018}.
    \item DANN: Domain Adversarial Neural Network \cite{dann2016}.
    \item CDANN: Conditional Domain Adversarial Neural Network \cite{cdann2018}.
    \item MTL: Marginal Transfer Learning \cite{mtl2021}.
    \item SagNet: Style Agnostic Networks \cite{sagnet2021}.
    \item ARM: Adaptive Risk Minimization \cite{arm2021}.
    \item V-REx: Variance Risk Extrapolation \cite{vrex2021}.
    \item RSC: Representation Self-Challenging \cite{rsc2020}.
    \item AND-mask: Learning Explanations that are Hard to Vary \cite{andmask2020}.
    \item SAND-mask: Smoothed-AND mask \cite{sandmask2021}.
    \item Fish: Gradient Matching for Domain Generalization \cite{fish2021}.
    \item Fishr: Invariant Gradient Variances for Out-of-distribution Generalization \cite{fishr2022}.
    \item SelfReg: Self-supervised Contrastive Regularization \cite{selfreg2021}.
    \item CausIRL: Invariant Causal Mechanisms through Distribution Matching \cite{causirl2022}.
\end{itemize}

The same fine-tuning procedure is applied to all approaches: The network is a multi-layer CNN for synthetic MNIST datasets and is a pre-trained ResNet-50 for other real-world datasets. The hyper-parameters are selected by a random search over $20$ independent trials for each target domain, and each evaluation score is reported after $3$ runs with different initialization seeds\footnote{\url{https://github.com/facebookresearch/DomainBed}}. The hyper-parameter selection criteria are shown in Table \ref{tbl:domainbed_param}. Note that warmup iterations and moving average techniques are not adopted for representation alignment.

\begin{table}[ht]
\centering
\small
\caption{The hyper-parameters of DomainBed.}
\label{tbl:domainbed_param}
\adjustbox{max width=\textwidth}{%
\begin{tabular}{llcl}
\toprule
Condition & Parameter & Default Value & Random Distribution \\
\midrule
\multirow{2}{*}{MNIST Datasets} & learning rate & $0.001$ & $10^{\mathrm{Uniform}(-4.5,-3.5)}$ \\
& batch size & $64$ & $2^{\mathrm{Uniform}(3,9)}$ \\
\midrule
\multirow{4}{*}{Real-world Datasets} & learning rate & $0.00005$ & $10^{\mathrm{Uniform}(-5,-3.5)}$ \\
& batch size & $32$ & $2^{\mathrm{Uniform}(3,5)}$ (DomainNet) / $2^{\mathrm{Uniform}(3,5.5)}$ (others) \\
& weight decay & $0$ & $10^{\mathrm{Uniform}(-6,-2)}$ \\
& dropout & $0$ & $\mathrm{Uniform}(\{0,0.1,0.5\})$ \\
\midrule
- & steps & $5000$ & $5000$ \\
\midrule
\multirow{5}{*}{IDM} & gradient penalty & $1000$ & $10^{\mathrm{Uniform}(1,5)}$ \\
& gradient warmup & $1500$ & $\mathrm{Uniform}(0, 5000)$ \\
& representation penalty & $1$ & $10^{\mathrm{Uniform}(-1,1)}$ \\
& moving average & $0.95$ & $\mathrm{Uniform}(0.9, 0.99)$ \\
\bottomrule
\end{tabular}}
\end{table}

Note that although the same Colored MNIST dataset is adopted by DomainBed, the experimental settings are completely different from the previous one \cite{irm2019}. The main difference is the batch size ($25000$ for IRM, less than $512$ for DomainBed), making it much harder to learn invariance for causality inference and distribution matching methods since fewer samples are available for probability density estimation. This explains the huge performance drop between these two experiments using the same DG algorithms.

\section*{Additional Experimental Results} \label{sec:results}

\subsection{Component Analysis}
In this section, we conduct ablation studies to demonstrate the effect of each component of the proposed IDM algorithm. Specifically, we analyze the effect of gradient alignment (GA), representation alignment (RA), warmup iterations (WU), moving average (MA), and the proposed PDM method for distribution matching.

\begin{table}[ht]
\centering
\small
\caption{Component Analysis on ColoredMNIST of DomainBed.}
\label{tbl:abla_cmnist}
\adjustbox{max width=\textwidth}{%
\begin{tabular}{lcccccccc}
\toprule
\textbf{Algorithm} & \textbf{GA} & \textbf{RA} & \textbf{WU} & \textbf{MA} & \textbf{90\%} & \textbf{80\%} & \textbf{10\%} & \textbf{Average} \\
\midrule
ERM & \multicolumn{4}{c}{-} & 71.8 $\pm$ 0.4 & 72.9 $\pm$ 0.1 & 28.7 $\pm$ 0.5 & 57.8 \\
\midrule
\multirow{5}{*}{IDM} & \xmark & \cmark & \xmark & \xmark & 71.9 $\pm$ 0.4 & 72.5 $\pm$ 0.0 & 28.8 $\pm$ 0.7 & 57.7 \\
& \cmark & \xmark & \cmark & \cmark & 73.1 $\pm$ 0.2 & 72.7 $\pm$ 0.3 & 67.4 $\pm$ 1.6 & 71.1 \\
& \cmark & \cmark & \xmark & \cmark & 72.9 $\pm$ 0.2 & 72.7 $\pm$ 0.1 & 60.8 $\pm$ 2.1 & 68.8 \\
& \cmark & \cmark & \cmark & \xmark & 72.0 $\pm$ 0.1 & 71.5 $\pm$ 0.3 & 48.7 $\pm$ 7.1 & 64.0 \\
\cmidrule{2-9}
& \cmark & \cmark & \cmark & \cmark & \textbf{74.2} $\pm$ 0.6 & \textbf{73.5} $\pm$ 0.2 & \textbf{68.3} $\pm$ 2.5 & \textbf{72.0} \\
\bottomrule
\end{tabular}}
\end{table}

\begin{table}[ht]
\centering
\small
\caption{Component Analysis on OfficeHome of DomainBed.}
\label{tbl:abla_officehome}
\adjustbox{max width=\textwidth}{%
\begin{tabular}{lccccccccc}
\toprule
\textbf{Algorithm} & \textbf{GA} & \textbf{RA} & \textbf{WU} & \textbf{MA} & \textbf{A} & \textbf{C} & \textbf{P} & \textbf{R} & \textbf{Average} \\
\midrule
ERM & \multicolumn{4}{c}{-} & 61.7 $\pm$ 0.7 & 53.4 $\pm$ 0.3 & 74.1 $\pm$ 0.4 & 76.2 $\pm$ 0.6 & 66.4 \\
\midrule
\multirow{5}{*}{IDM} & \xmark & \cmark & \xmark & \xmark & \textbf{64.7} $\pm$ 0.5 & \textbf{54.6} $\pm$ 0.3 & 76.2 $\pm$ 0.4 & \textbf{78.1} $\pm$ 0.5 & \textbf{68.4} \\
& \cmark & \xmark & \cmark & \cmark & 61.9 $\pm$ 0.4 & 53.0 $\pm$ 0.3 & 75.5 $\pm$ 0.2 & 77.9 $\pm$ 0.2 & 67.1 \\
& \cmark & \cmark & \xmark & \cmark & 62.5 $\pm$ 0.1 & 53.0 $\pm$ 0.7 & 75.0 $\pm$ 0.4 & 77.2 $\pm$ 0.7 & 66.9 \\
& \cmark & \cmark & \cmark & \xmark & 64.2 $\pm$ 0.3 & 53.5 $\pm$ 0.6 & 76.1 $\pm$ 0.4 & \textbf{78.1} $\pm$ 0.4 & 68.0 \\
\cmidrule{2-10}
& \cmark & \cmark & \cmark & \cmark & 64.4 $\pm$ 0.3 & 54.4 $\pm$ 0.6 & \textbf{76.5} $\pm$ 0.3 & 78.0 $\pm$ 0.4 & 68.3 \\
\bottomrule
\end{tabular}}
\end{table}

\subsubsection{Gradient Alignment}
According to our theoretical analysis, gradient alignment promotes training-domain generalization, especially when concept shift is prominent. As can be seen in Table \ref{tbl:abla_cmnist}, IDM without gradient alignment (57.7\%) performs similarly to ERM (57.8\%), which is unable to learn invariance across training domains. Gradient alignment also significantly boosts the performance on VLCS (77.4\% to 78.1\%) and PACS (86.8\% to 87.6\%), as seen in Table \ref{tbl:abla_ga1} and \ref{tbl:abla_ga2}. However, for datasets where concept shift is not prominent e.g. OfficeHome, gradient alignment cannot help to improve performance as shown in Table \ref{tbl:abla_officehome}. It is worth noting that gradient alignment also penalizes a lower bound for the representation space distribution shift: In the $t$-th step of gradient descent, the Markov chain relationship $D_i \rightarrow B_t^i \rightarrow (R_t^i, Y_t^i) \rightarrow G_t^i$ holds conditioned on the current predictor $W_{t-1}$, which implies the lower bound $I(G_t^i;D_i|W_{t-1}) \le I(R_t^i, Y_t^i; D_i|W_{t-1})$ by the data processing inequality. This indicates that gradient alignment also helps to address the covariate shift, which explains the promising performance of gradient-based DG algorithms e.g. Fish and Fishr. However, since this is a lower bound rather than an upper bound, gradient manipulation is insufficient to fully address representation space covariate shifts, as seen in the following analysis for representation alignment.

\begin{table}[ht]
\centering
\small
\caption{Effect of gradient alignment (GA) on VLCS of DomainBed.}
\label{tbl:abla_ga1}
\adjustbox{max width=\textwidth}{%
\begin{tabular}{lcccccc}
\toprule
\textbf{Algorithm} & \textbf{GA} & \textbf{A} & \textbf{C} & \textbf{P} & \textbf{S} & \textbf{Average} \\
\midrule
ERM & - & \textbf{97.6} $\pm$ 0.3 & \textbf{67.9} $\pm$ 0.7 & 70.9 $\pm$ 0.2 & 74.0 $\pm$ 0.6 & 77.6 \\
IDM & \xmark & 97.1 $\pm$ 0.7 & 67.2 $\pm$ 0.4 & 69.9 $\pm$ 0.4 & 75.6 $\pm$ 0.8 & 77.4 \\
IDM & \cmark & \textbf{97.6} $\pm$ 0.3 & 66.9 $\pm$ 0.3 & \textbf{71.8} $\pm$ 0.5 & \textbf{76.0} $\pm$ 1.3 & \textbf{78.1} \\
\bottomrule
\end{tabular}}
\end{table}

\begin{table}[ht]
\centering
\small
\caption{Effect of gradient alignment (GA) on PACS of DomainBed.}
\label{tbl:abla_ga2}
\adjustbox{max width=\textwidth}{%
\begin{tabular}{lcccccc}
\toprule
\textbf{Algorithm} & \textbf{GA} & \textbf{A} & \textbf{C} & \textbf{P} & \textbf{S} & \textbf{Average} \\
\midrule
ERM & - & 86.5 $\pm$ 1.0 & 81.3 $\pm$ 0.6 & 96.2 $\pm$ 0.3 & \textbf{82.7} $\pm$ 1.1 & 86.7 \\
IDM & \xmark & 87.8 $\pm$ 0.6 & 81.6 $\pm$ 0.3 & 97.4 $\pm$ 0.2 & 80.6 $\pm$ 1.3 & 86.8 \\
IDM & \cmark & \textbf{88.0} $\pm$ 0.3 & \textbf{82.6} $\pm$ 0.6 & \textbf{97.6} $\pm$ 0.4 & 82.3 $\pm$ 0.6 & \textbf{87.6} \\
\bottomrule
\end{tabular}}
\end{table}

\subsubsection{Representation Alignment}
Representation alignment promotes test-domain generalization by minimizing the representation level covariate shift. As shown in Table \ref{tbl:abla_cmnist} - \ref{tbl:abla_ra1}, representation alignment is effective in OfficeHome (67.1\% to 68.3\%) and RotatedMNIST (97.8\% to 98.0\%), and still enhances the performance even though covariate shift is not prominent in ColoredMNIST (71.1\% to 72.0\%). This verifies our claim that representation alignment complements gradient alignment in solving Problem \ref{prb:dg_hp}, and is necessary for achieving high-probability DG.

\begin{table}[ht]
\centering
\small
\caption{Effect of representation alignment (RA) on RotatedMNIST of DomainBed.}
\label{tbl:abla_ra1}
\adjustbox{max width=\textwidth}{%
\begin{tabular}{lcccccccc}
\toprule
\textbf{Algorithm} & \textbf{RA} & \textbf{0} & \textbf{15} & \textbf{30} & \textbf{45} & \textbf{60} & \textbf{75} & \textbf{Average} \\
\midrule
ERM & - & 95.3 $\pm$ 0.2 & \textbf{98.7} $\pm$ 0.1 & 98.9 $\pm$ 0.1 & 98.7 $\pm$ 0.2 & \textbf{98.9} $\pm$ 0.0 & 96.2 $\pm$ 0.2 & 97.8 \\
IDM & \xmark & 95.6 $\pm$ 0.1 & 98.4 $\pm$ 0.1 & 98.7 $\pm$ 0.2 & \textbf{99.1} $\pm$ 0.0 & 98.7 $\pm$ 0.1 & \textbf{96.6} $\pm$ 0.4 & 97.8 \\
IDM & \cmark & \textbf{96.1} $\pm$ 0.3 & \textbf{98.7} $\pm$ 0.1 & \textbf{99.1} $\pm$ 0.1 & 98.9 $\pm$ 0.1 & \textbf{98.9} $\pm$ 0.1 & \textbf{96.6} $\pm$ 0.1 & \textbf{98.0} \\
\bottomrule
\end{tabular}}
\end{table}

\subsubsection{Warmup Iterations}
Following the experimental settings of \cite{irm2019, fishr2022}, we do not apply the penalties of gradient or representation alignment until the number of epochs reaches a certain value. This is inspired by the observation that forcing invariance in early steps may hinder the models from extracting useful correlations. By incorporating these warmup iterations, predictors are allowed to extract all possible correlations between the inputs and the labels at the beginning, and then discard spurious ones in later updates. As can be seen in Table \ref{tbl:abla_cmnist} and \ref{tbl:abla_officehome}, this strategy helps to enhance the final performances on ColoredMNIST (68.8\% to 72.0\%) and OfficeHome (66.9\% to 68.3\%).

\subsubsection{Moving Average}
Following \cite{fishr2022, pooladzandi2022adaptive}, we use an exponential moving average when computing the gradients or the representations. This strategy helps when the batch size is not sufficiently large to sketch the probability distributions. In the IRM experiment setup where the batch size is $25000$, Fishr (70.2\%) and IDM (70.5\%) both easily achieve near-optimal accuracy compared to Oracle (71.0\%). In the DomainBed setup, the batch size $2^{\mathrm{Uniform}(3,9)}$ is significantly diminished, resulting in worse test-domain accuracy of Fishr (68.8\%). As shown in Table \ref{tbl:abla_cmnist} and \ref{tbl:abla_officehome}, this moving average strategy greatly enhances the performance of IDM on ColoredMNIST (64.0\% to 72.0\%) and OfficeHome (68.0\% to 68.3\%).

\subsubsection{PDM for Distribution Matching}
We then demonstrate the superiority of our PDM method over moment-based distribution alignment techniques. Specifically, we compare IGA \cite{iga2020} which matches the empirical expectation of the gradients, Fishr \cite{fishr2022} which proposes to align the gradient variance, the combination of IGA + Fishr (i.e. aligning the expectation and variance simultaneously), and our approach IDM (without representation space alignment). The performance gain of IDM on the Colored MNIST task in \cite{irm2019} is not significant, since it is relatively easier to learn invariance with a large batch size ($25000$). In the DomainBed setting, the batch size is significantly reduced ($8$-$512$), making this learning task much harder. The results are reported in Table \ref{tbl:abla_pdm}.

\begin{table}[ht]
\centering
\small
\caption{Superiority of PDM on Colored MNIST of DomainBed.}
\label{tbl:abla_pdm}
\adjustbox{max width=\textwidth}{%
\begin{tabular}{lcccc}
\toprule
\textbf{Algorithm} & \textbf{90\%} & \textbf{80\%} & \textbf{10\%} & \textbf{Average} \\
\midrule
ERM & 71.8 $\pm$ 0.4 & 72.9 $\pm$ 0.1 & 28.7 $\pm$ 0.5 & 57.8 \\
IGA & 72.6 $\pm$ 0.3 & 72.9 $\pm$ 0.2 & 50.0 $\pm$ 1.2 & 65.2 \\
Fishr & 74.1 $\pm$ 0.6 & 73.3 $\pm$ 0.1 & 58.9 $\pm$ 3.7 & 68.8 \\
IGA + Fishr & 73.3 $\pm$ 0.0 & 72.6 $\pm$ 0.5 & 66.3 $\pm$ 2.9 & 70.7 \\
\midrule
IDM & \textbf{74.2} $\pm$ 0.6 & \textbf{73.5} $\pm$ 0.2 & \textbf{68.3} $\pm$ 2.5 & \textbf{72.0} \\
\bottomrule
\end{tabular}}
\end{table}

As can be seen, IDM achieves significantly higher performance on Colored MNIST (72.0\%) even compared to the combination of IGA + Fishr (70.7\%). This verifies our conclusion that matching the expectation and the variance is not sufficient for complex probability distributions, and demonstrates the superiority of the proposed PDM method for distribution alignment.

\subsection{Running Time Comparison}
Since IDM only stores historical gradients and representations for a single batch from each training domain, the storage and computation overhead is marginal compared to training the entire network. As shown in Table \ref{tbl:running_time}, the training time is only 5\% longer compared to ERM on the largest DomainNet dataset.

\begin{table}
\small
    \centering
    \caption{Computational overhead of IDM using default batch size.}
    \label{tbl:running_time}
    \begin{tabular}{lcccccc}
        \toprule
        \multirow{2}{*}{Dataset} & \multicolumn{3}{c}{Training Time (h)} & \multicolumn{3}{c}{Memory Requirement (GB)} \\
        \cmidrule(l{2pt}r{2pt}){2-4} \cmidrule(l{2pt}r{2pt}){5-7}
        & ERM & IDM & Overhead & ERM & IDM & Overhead \\
        \midrule
        ColoredMNIST & 0.076 & 0.088 & 14.6\% & 0.138 & 0.139 & 0.2\% \\
        RotatedMNIST & 0.101 & 0.110 & 9.3\% & 0.338 & 0.342 & 1.0\%\\
        VLCS & 0.730 & 0.744 & 2.0\% & 8.189 & 8.199 & 0.1\% \\
        PACS & 0.584 & 0.593 & 1.5\% & 8.189 & 8.201 & 0.1\% \\
        OfficeHome & 0.690 & 0.710 & 2.9\% & 8.191 & 8.506 & 3.8\% \\
        TerraIncognita & 0.829 & 0.840 & 1.3\% & 8.189 & 8.208 & 0.2\% \\
        DomainNet & 2.805 & 2.947 & 5.0\% & 13.406 & 16.497 & 23.1\% \\
        \bottomrule
    \end{tabular}
\end{table}

\subsection{Full DomainBed Results} \label{sec:full_domainbed}

In this paper, we focus on the test-domain model selection criterion, where the validation set follows the same distribution as the test domains. Our choice is well-motivated for the following reasons:
\begin{itemize}
    \item Test-domain validation is provided by the DomainBed benchmark as one of the default model-selection methods, and is also widely adopted in the literature in many significant works like IRM \cite{irm2019}, V-Rex \cite{vrex2021}, and Fishr \cite{fishr2022}.
    \item As suggested by Proposition \ref{prop:concept_shift}, any algorithm that fits well on training domains will suffer from strictly positive risks in test domains once concept shift is induced. Therefore, training-domain validation would result in sub-optimal selection results.
    \item Training-domain validation may render efforts to address concept shift useless, as spurious features are often more predictive than invariant ones. This is particularly unfair for algorithms that aim to tackle the concept shift. As shown in Table 9 in \cite{fishr2022}, no algorithm can significantly outperform ERM on Colored MNIST using training-domain validation (an exception is ARM which uses test-time adaptation, and thus cannot be directly compared), even though ERM is shown to perform much worse than random guessing (10\% v.s. 50\% accuracy) for the last domain (see Table 1 in \cite{irm2019} and Appendix D.4.1 in \cite{fishr2022}). As a result, models selected by training-domain validation may not generalize well when concept shift is substantial.
    \item As mentioned by \cite{d2022underspecification}, training-domain validation suffers from underspecification, where predictors with equivalently strong performances in training domains may behave very differently during testing. It is also emphasized by \cite{teney2022evading} that OOD performance cannot, by definition, be performed with a validation set from the same distribution as the training data. This further raises concerns about the validity of using training-domain accuracies for validation purposes.
    \item Moreover, test-domain validation is also applicable in practice, as it is feasible to label a few test-domain samples for validation purposes. It is also unrealistic to deploy models in target environments without any form of verification, making such efforts necessary in practice.
\end{itemize}

Due to the reasons listed above, we believe that the test-domain validation results are sufficient to demonstrate the effectiveness of our approach in real-world learning scenarios. We report detailed results of IDM for each domain in each dataset of the DomainBed benchmark under test-domain model selection for a complete evaluation in Table \ref{tbl:domainbed_cmnist} - \ref{tbl:domainbed_domainnet}. Note that detailed scores of certain algorithms (Fish, CausIRL) are not available.

\begin{table}[htbp]
\centering
\small
\caption{Detailed results on Colored MNIST in DomainBed.}
\label{tbl:domainbed_cmnist}
\adjustbox{max width=\textwidth}{%
\begin{tabular}{lcccc}
\toprule
\textbf{Algorithm} & \textbf{90\%} & \textbf{80\%} & \textbf{10\%} & \textbf{Average} \\
\midrule
ERM & 71.8 $\pm$ 0.4 & 72.9 $\pm$ 0.1 & 28.7 $\pm$ 0.5 & 57.8 \\
IRM & 72.0 $\pm$ 0.1 & 72.5 $\pm$ 0.3 & 58.5 $\pm$ 3.3 & 67.7 \\
GroupDRO & 73.5 $\pm$ 0.3 & 73.0 $\pm$ 0.3 & 36.8 $\pm$ 2.8 & 61.1 \\
Mixup & 72.5 $\pm$ 0.2 & 73.9 $\pm$ 0.4 & 28.6 $\pm$ 0.2 & 58.4 \\
MLDG & 71.9 $\pm$ 0.3 & 73.5 $\pm$ 0.2 & 29.1 $\pm$ 0.9 & 58.2 \\
CORAL & 71.1 $\pm$ 0.2 & 73.4 $\pm$ 0.2 & 31.1 $\pm$ 1.6 & 58.6 \\
MMD & 69.0 $\pm$ 2.3 & 70.4 $\pm$ 1.6 & 50.6 $\pm$ 0.2 & 63.3 \\
DANN & 72.4 $\pm$ 0.5 & 73.9 $\pm$ 0.5 & 24.9 $\pm$ 2.7 & 57.0 \\
CDANN & 71.8 $\pm$ 0.5 & 72.9 $\pm$ 0.1 & 33.8 $\pm$ 6.4 & 59.5 \\
MTL & 71.2 $\pm$ 0.2 & 73.5 $\pm$ 0.2 & 28.0 $\pm$ 0.6 & 57.6 \\
SagNet & 72.1 $\pm$ 0.3 & 73.2 $\pm$ 0.3 & 29.4 $\pm$ 0.5 & 58.2 \\
ARM & 84.9 $\pm$ 0.9 & 76.8 $\pm$ 0.6 & 27.9 $\pm$ 2.1 & 63.2 \\
V-REx & 72.8 $\pm$ 0.3 & 73.0 $\pm$ 0.3 & 55.2 $\pm$ 4.0 & 67.0 \\
RSC & 72.0 $\pm$ 0.1 & 73.2 $\pm$ 0.1 & 30.2 $\pm$ 1.6 & 58.5 \\
AND-mask & 71.9 $\pm$ 0.6 & 73.6 $\pm$ 0.5 & 30.2 $\pm$ 1.4 & 58.6 \\
SAND-mask & 79.9 $\pm$ 3.8 & 75.9 $\pm$ 1.6 & 31.6 $\pm$ 1.1 & 62.3 \\
Fishr & 74.1 $\pm$ 0.6 & 73.3 $\pm$ 0.1 & 58.9 $\pm$ 3.7 & 68.8 \\
SelfReg & 71.3 $\pm$ 0.4 & 73.4 $\pm$ 0.2 & 29.3 $\pm$ 2.1 & 58.0 \\
\midrule
IDM & 74.2 $\pm$ 0.6 & 73.5 $\pm$ 0.2 & 68.3 $\pm$ 2.5 & 72.0 \\
\bottomrule
\end{tabular}}
\end{table}

\begin{table}[htbp]
\centering
\small
\caption{Detailed results on Rotated MNIST in DomainBed.}
\label{tbl:domainbed_rmnist}
\adjustbox{max width=\textwidth}{%
\begin{tabular}{lccccccc}
\toprule
\textbf{Algorithm} & \textbf{0} & \textbf{15} & \textbf{30} & \textbf{45} & \textbf{60} & \textbf{75} & \textbf{Average} \\
\midrule
ERM & 95.3 $\pm$ 0.2 & 98.7 $\pm$ 0.1 & 98.9 $\pm$ 0.1 & 98.7 $\pm$ 0.2 & 98.9 $\pm$ 0.0 & 96.2 $\pm$ 0.2 & 97.8 \\
IRM & 94.9 $\pm$ 0.6 & 98.7 $\pm$ 0.2 & 98.6 $\pm$ 0.1 & 98.6 $\pm$ 0.2 & 98.7 $\pm$ 0.1 & 95.2 $\pm$ 0.3 & 97.5 \\
GroupDRO & 95.9 $\pm$ 0.1 & 99.0 $\pm$ 0.1 & 98.9 $\pm$ 0.1 & 98.8 $\pm$ 0.1 & 98.6 $\pm$ 0.1 & 96.3 $\pm$ 0.4 & 97.9 \\
Mixup & 95.8 $\pm$ 0.3 & 98.7 $\pm$ 0.0 & 99.0 $\pm$ 0.1 & 98.8 $\pm$ 0.1 & 98.8 $\pm$ 0.1 & 96.6 $\pm$ 0.2 & 98.0 \\
MLDG & 95.7 $\pm$ 0.2 & 98.9 $\pm$ 0.1 & 98.8 $\pm$ 0.1 & 98.9 $\pm$ 0.1 & 98.6 $\pm$ 0.1 & 95.8 $\pm$ 0.4 & 97.8 \\
CORAL & 96.2 $\pm$ 0.2 & 98.8 $\pm$ 0.1 & 98.8 $\pm$ 0.1 & 98.8 $\pm$ 0.1 & 98.9 $\pm$ 0.1 & 96.4 $\pm$ 0.2 & 98.0 \\
MMD & 96.1 $\pm$ 0.2 & 98.9 $\pm$ 0.0 & 99.0 $\pm$ 0.0 & 98.8 $\pm$ 0.0 & 98.9 $\pm$ 0.0 & 96.4 $\pm$ 0.2 & 98.0 \\
DANN & 95.9 $\pm$ 0.1 & 98.9 $\pm$ 0.1 & 98.6 $\pm$ 0.2 & 98.7 $\pm$ 0.1 & 98.9 $\pm$ 0.0 & 96.3 $\pm$ 0.3 & 97.9 \\
CDANN & 95.9 $\pm$ 0.2 & 98.8 $\pm$ 0.0 & 98.7 $\pm$ 0.1 & 98.9 $\pm$ 0.1 & 98.8 $\pm$ 0.1 & 96.1 $\pm$ 0.3 & 97.9 \\
MTL & 96.1 $\pm$ 0.2 & 98.9 $\pm$ 0.0 & 99.0 $\pm$ 0.0 & 98.7 $\pm$ 0.1 & 99.0 $\pm$ 0.0 & 95.8 $\pm$ 0.3 & 97.9 \\
SagNet & 95.9 $\pm$ 0.1 & 99.0 $\pm$ 0.1 & 98.9 $\pm$ 0.1 & 98.6 $\pm$ 0.1 & 98.8 $\pm$ 0.1 & 96.3 $\pm$ 0.1 & 97.9 \\
ARM & 95.9 $\pm$ 0.4 & 99.0 $\pm$ 0.1 & 98.8 $\pm$ 0.1 & 98.9 $\pm$ 0.1 & 99.1 $\pm$ 0.1 & 96.7 $\pm$ 0.2 & 98.1 \\
V-REx & 95.5 $\pm$ 0.2 & 99.0 $\pm$ 0.0 & 98.7 $\pm$ 0.2 & 98.8 $\pm$ 0.1 & 98.8 $\pm$ 0.0 & 96.4 $\pm$ 0.0 & 97.9 \\
RSC & 95.4 $\pm$ 0.1 & 98.6 $\pm$ 0.1 & 98.6 $\pm$ 0.1 & 98.9 $\pm$ 0.0 & 98.8 $\pm$ 0.1 & 95.4 $\pm$ 0.3 & 97.6 \\
AND-mask & 94.9 $\pm$ 0.1 & 98.8 $\pm$ 0.1 & 98.8 $\pm$ 0.1 & 98.7 $\pm$ 0.2 & 98.6 $\pm$ 0.2 & 95.5 $\pm$ 0.2 & 97.5 \\
SAND-mask & 94.7 $\pm$ 0.2 & 98.5 $\pm$ 0.2 & 98.6 $\pm$ 0.1 & 98.6 $\pm$ 0.1 & 98.5 $\pm$ 0.1 & 95.2 $\pm$ 0.1 & 97.4 \\
Fishr & 95.8 $\pm$ 0.1 & 98.3 $\pm$ 0.1 & 98.8 $\pm$ 0.1 & 98.6 $\pm$ 0.3 & 98.7 $\pm$ 0.1 & 96.5 $\pm$ 0.1 & 97.8 \\
SelfReg & 96.0 $\pm$ 0.3 & 98.9 $\pm$ 0.1 & 98.9 $\pm$ 0.1 & 98.9 $\pm$ 0.1 & 98.9 $\pm$ 0.1 & 96.8 $\pm$ 0.1 & 98.1 \\
\midrule
IDM & 96.1 $\pm$ 0.3 & 98.7 $\pm$ 0.1 & 99.1 $\pm$ 0.1 & 98.9 $\pm$ 0.1 & 98.9 $\pm$ 0.1 & 96.6 $\pm$ 0.1 & 98.0 \\
\bottomrule
\end{tabular}}
\end{table}

\begin{table}[htbp]
\centering
\small
\caption{Detailed results on VLCS in DomainBed.}
\label{tbl:domainbed_vlcs}
\adjustbox{max width=\textwidth}{%
\begin{tabular}{lccccc}
\toprule
\textbf{Algorithm} & \textbf{C} & \textbf{L} & \textbf{S} & \textbf{V} & \textbf{Average} \\
\midrule
ERM & 97.6 $\pm$ 0.3 & 67.9 $\pm$ 0.7 & 70.9 $\pm$ 0.2 & 74.0 $\pm$ 0.6 & 77.6 \\
IRM & 97.3 $\pm$ 0.2 & 66.7 $\pm$ 0.1 & 71.0 $\pm$ 2.3 & 72.8 $\pm$ 0.4 & 76.9 \\
GroupDRO & 97.7 $\pm$ 0.2 & 65.9 $\pm$ 0.2 & 72.8 $\pm$ 0.8 & 73.4 $\pm$ 1.3 & 77.4 \\
Mixup & 97.8 $\pm$ 0.4 & 67.2 $\pm$ 0.4 & 71.5 $\pm$ 0.2 & 75.7 $\pm$ 0.6 & 78.1 \\
MLDG & 97.1 $\pm$ 0.5 & 66.6 $\pm$ 0.5 & 71.5 $\pm$ 0.1 & 75.0 $\pm$ 0.9 & 77.5 \\
CORAL & 97.3 $\pm$ 0.2 & 67.5 $\pm$ 0.6 & 71.6 $\pm$ 0.6 & 74.5 $\pm$ 0.0 & 77.7 \\
MMD & 98.8 $\pm$ 0.0 & 66.4 $\pm$ 0.4 & 70.8 $\pm$ 0.5 & 75.6 $\pm$ 0.4 & 77.9 \\
DANN & 99.0 $\pm$ 0.2 & 66.3 $\pm$ 1.2 & 73.4 $\pm$ 1.4 & 80.1 $\pm$ 0.5 & 79.7 \\
CDANN & 98.2 $\pm$ 0.1 & 68.8 $\pm$ 0.5 & 74.3 $\pm$ 0.6 & 78.1 $\pm$ 0.5 & 79.9 \\
MTL & 97.9 $\pm$ 0.7 & 66.1 $\pm$ 0.7 & 72.0 $\pm$ 0.4 & 74.9 $\pm$ 1.1 & 77.7 \\
SagNet & 97.4 $\pm$ 0.3 & 66.4 $\pm$ 0.4 & 71.6 $\pm$ 0.1 & 75.0 $\pm$ 0.8 & 77.6 \\
ARM & 97.6 $\pm$ 0.6 & 66.5 $\pm$ 0.3 & 72.7 $\pm$ 0.6 & 74.4 $\pm$ 0.7 & 77.8 \\
V-REx & 98.4 $\pm$ 0.2 & 66.4 $\pm$ 0.7 & 72.8 $\pm$ 0.1 & 75.0 $\pm$ 1.4 & 78.1 \\
RSC & 98.0 $\pm$ 0.4 & 67.2 $\pm$ 0.3 & 70.3 $\pm$ 1.3 & 75.6 $\pm$ 0.4 & 77.8 \\
AND-mask & 98.3 $\pm$ 0.3 & 64.5 $\pm$ 0.2 & 69.3 $\pm$ 1.3 & 73.4 $\pm$ 1.3 & 76.4 \\
SAND-mask & 97.6 $\pm$ 0.3 & 64.5 $\pm$ 0.6 & 69.7 $\pm$ 0.6 & 73.0 $\pm$ 1.2 & 76.2 \\
Fishr & 97.6 $\pm$ 0.7 & 67.3 $\pm$ 0.5 & 72.2 $\pm$ 0.9 & 75.7 $\pm$ 0.3 & 78.2 \\
SelfReg & 97.9 $\pm$ 0.4 & 66.7 $\pm$ 0.1 & 73.5 $\pm$ 0.7 & 74.7 $\pm$ 0.7 & 78.2 \\
\midrule
IDM & 97.6 $\pm$ 0.3 & 66.9 $\pm$ 0.3 & 71.8 $\pm$ 0.5 & 76.0 $\pm$ 1.3 & 78.1 \\
\bottomrule
\end{tabular}}
\end{table}

\begin{table}[htbp]
\centering
\small
\caption{Detailed results on PACS in DomainBed.}
\label{tbl:domainbed_pacs}
\adjustbox{max width=\textwidth}{%
\begin{tabular}{lccccc}
\toprule
\textbf{Algorithm} & \textbf{A} & \textbf{C} & \textbf{P} & \textbf{S} & \textbf{Average} \\
\midrule
ERM & 86.5 $\pm$ 1.0 & 81.3 $\pm$ 0.6 & 96.2 $\pm$ 0.3 & 82.7 $\pm$ 1.1 & 86.7 \\
IRM & 84.2 $\pm$ 0.9 & 79.7 $\pm$ 1.5 & 95.9 $\pm$ 0.4 & 78.3 $\pm$ 2.1 & 84.5 \\
GroupDRO & 87.5 $\pm$ 0.5 & 82.9 $\pm$ 0.6 & 97.1 $\pm$ 0.3 & 81.1 $\pm$ 1.2 & 87.1 \\
Mixup & 87.5 $\pm$ 0.4 & 81.6 $\pm$ 0.7 & 97.4 $\pm$ 0.2 & 80.8 $\pm$ 0.9 & 86.8 \\
MLDG & 87.0 $\pm$ 1.2 & 82.5 $\pm$ 0.9 & 96.7 $\pm$ 0.3 & 81.2 $\pm$ 0.6 & 86.8 \\
CORAL & 86.6 $\pm$ 0.8 & 81.8 $\pm$ 0.9 & 97.1 $\pm$ 0.5 & 82.7 $\pm$ 0.6 & 87.1 \\
MMD & 88.1 $\pm$ 0.8 & 82.6 $\pm$ 0.7 & 97.1 $\pm$ 0.5 & 81.2 $\pm$ 1.2 & 87.2 \\
DANN & 87.0 $\pm$ 0.4 & 80.3 $\pm$ 0.6 & 96.8 $\pm$ 0.3 & 76.9 $\pm$ 1.1 & 85.2 \\
CDANN & 87.7 $\pm$ 0.6 & 80.7 $\pm$ 1.2 & 97.3 $\pm$ 0.4 & 77.6 $\pm$ 1.5 & 85.8 \\
MTL & 87.0 $\pm$ 0.2 & 82.7 $\pm$ 0.8 & 96.5 $\pm$ 0.7 & 80.5 $\pm$ 0.8 & 86.7 \\
SagNet & 87.4 $\pm$ 0.5 & 81.2 $\pm$ 1.2 & 96.3 $\pm$ 0.8 & 80.7 $\pm$ 1.1 & 86.4 \\
ARM & 85.0 $\pm$ 1.2 & 81.4 $\pm$ 0.2 & 95.9 $\pm$ 0.3 & 80.9 $\pm$ 0.5 & 85.8 \\
V-REx & 87.8 $\pm$ 1.2 & 81.8 $\pm$ 0.7 & 97.4 $\pm$ 0.2 & 82.1 $\pm$ 0.7 & 87.2 \\
RSC & 86.0 $\pm$ 0.7 & 81.8 $\pm$ 0.9 & 96.8 $\pm$ 0.7 & 80.4 $\pm$ 0.5 & 86.2 \\
AND-mask & 86.4 $\pm$ 1.1 & 80.8 $\pm$ 0.9 & 97.1 $\pm$ 0.2 & 81.3 $\pm$ 1.1 & 86.4 \\
SAND-mask & 86.1 $\pm$ 0.6 & 80.3 $\pm$ 1.0 & 97.1 $\pm$ 0.3 & 80.0 $\pm$ 1.3 & 85.9 \\
Fishr & 87.9 $\pm$ 0.6 & 80.8 $\pm$ 0.5 & 97.9 $\pm$ 0.4 & 81.1 $\pm$ 0.8 & 86.9 \\
SelfReg & 87.5 $\pm$ 0.1 & 83.0 $\pm$ 0.1 & 97.6 $\pm$ 0.1 & 82.8 $\pm$ 0.2 & 87.7 \\
\midrule
IDM & 88.0 $\pm$ 0.3 & 82.6 $\pm$ 0.6 & 97.6 $\pm$ 0.4 & 82.3 $\pm$ 0.6 & 87.6 \\
\bottomrule
\end{tabular}}
\end{table}

\begin{table}[htbp]
\centering
\small
\caption{Detailed results on OfficeHome in DomainBed.}
\label{tbl:domainbed_officehome}
\adjustbox{max width=\textwidth}{%
\begin{tabular}{lccccc}
\toprule
\textbf{Algorithm} & \textbf{A} & \textbf{C} & \textbf{P} & \textbf{R} & \textbf{Average} \\
\midrule
ERM & 61.7 $\pm$ 0.7 & 53.4 $\pm$ 0.3 & 74.1 $\pm$ 0.4 & 76.2 $\pm$ 0.6 & 66.4 \\
IRM & 56.4 $\pm$ 3.2 & 51.2 $\pm$ 2.3 & 71.7 $\pm$ 2.7 & 72.7 $\pm$ 2.7 & 63.0 \\
GroupDRO & 60.5 $\pm$ 1.6 & 53.1 $\pm$ 0.3 & 75.5 $\pm$ 0.3 & 75.9 $\pm$ 0.7 & 66.2 \\
Mixup & 63.5 $\pm$ 0.2 & 54.6 $\pm$ 0.4 & 76.0 $\pm$ 0.3 & 78.0 $\pm$ 0.7 & 68.0 \\
MLDG & 60.5 $\pm$ 0.7 & 54.2 $\pm$ 0.5 & 75.0 $\pm$ 0.2 & 76.7 $\pm$ 0.5 & 66.6 \\
CORAL & 64.8 $\pm$ 0.8 & 54.1 $\pm$ 0.9 & 76.5 $\pm$ 0.4 & 78.2 $\pm$ 0.4 & 68.4 \\
MMD & 60.4 $\pm$ 1.0 & 53.4 $\pm$ 0.5 & 74.9 $\pm$ 0.1 & 76.1 $\pm$ 0.7 & 66.2 \\
DANN & 60.6 $\pm$ 1.4 & 51.8 $\pm$ 0.7 & 73.4 $\pm$ 0.5 & 75.5 $\pm$ 0.9 & 65.3 \\
CDANN & 57.9 $\pm$ 0.2 & 52.1 $\pm$ 1.2 & 74.9 $\pm$ 0.7 & 76.2 $\pm$ 0.2 & 65.3 \\
MTL & 60.7 $\pm$ 0.8 & 53.5 $\pm$ 1.3 & 75.2 $\pm$ 0.6 & 76.6 $\pm$ 0.6 & 66.5 \\
SagNet & 62.7 $\pm$ 0.5 & 53.6 $\pm$ 0.5 & 76.0 $\pm$ 0.3 & 77.8 $\pm$ 0.1 & 67.5 \\
ARM & 58.8 $\pm$ 0.5 & 51.8 $\pm$ 0.7 & 74.0 $\pm$ 0.1 & 74.4 $\pm$ 0.2 & 64.8 \\
V-REx & 59.6 $\pm$ 1.0 & 53.3 $\pm$ 0.3 & 73.2 $\pm$ 0.5 & 76.6 $\pm$ 0.4 & 65.7 \\
RSC & 61.7 $\pm$ 0.8 & 53.0 $\pm$ 0.9 & 74.8 $\pm$ 0.8 & 76.3 $\pm$ 0.5 & 66.5 \\
AND-mask & 60.3 $\pm$ 0.5 & 52.3 $\pm$ 0.6 & 75.1 $\pm$ 0.2 & 76.6 $\pm$ 0.3 & 66.1 \\
SAND-mask & 59.9 $\pm$ 0.7 & 53.6 $\pm$ 0.8 & 74.3 $\pm$ 0.4 & 75.8 $\pm$ 0.5 & 65.9 \\
Fishr & 63.4 $\pm$ 0.8 & 54.2 $\pm$ 0.3 & 76.4 $\pm$ 0.3 & 78.5 $\pm$ 0.2 & 68.2 \\
SelfReg & 64.2 $\pm$ 0.6 & 53.6 $\pm$ 0.7 & 76.7 $\pm$ 0.3 & 77.9 $\pm$ 0.5 & 68.1 \\
\midrule
IDM & 64.4 $\pm$ 0.3 & 54.4 $\pm$ 0.6 & 76.5 $\pm$ 0.3 & 78.0 $\pm$ 0.4 & 68.3 \\
\bottomrule
\end{tabular}}
\end{table}

\begin{table}[htbp]
\centering
\small
\caption{Detailed results on TerraIncognita in DomainBed.}
\label{tbl:domainbed_terra}
\adjustbox{max width=\textwidth}{%
\begin{tabular}{lccccc}
\toprule
\textbf{Algorithm} & \textbf{L100} & \textbf{L38} & \textbf{L43} & \textbf{L46} & \textbf{Average} \\
\midrule
ERM & 59.4 $\pm$ 0.9 & 49.3 $\pm$ 0.6 & 60.1 $\pm$ 1.1 & 43.2 $\pm$ 0.5 & 53.0 \\
IRM & 56.5 $\pm$ 2.5 & 49.8 $\pm$ 1.5 & 57.1 $\pm$ 2.2 & 38.6 $\pm$ 1.0 & 50.5 \\
GroupDRO & 60.4 $\pm$ 1.5 & 48.3 $\pm$ 0.4 & 58.6 $\pm$ 0.8 & 42.2 $\pm$ 0.8 & 52.4 \\
Mixup & 67.6 $\pm$ 1.8 & 51.0 $\pm$ 1.3 & 59.0 $\pm$ 0.0 & 40.0 $\pm$ 1.1 & 54.4 \\
MLDG & 59.2 $\pm$ 0.1 & 49.0 $\pm$ 0.9 & 58.4 $\pm$ 0.9 & 41.4 $\pm$ 1.0 & 52.0 \\
CORAL & 60.4 $\pm$ 0.9 & 47.2 $\pm$ 0.5 & 59.3 $\pm$ 0.4 & 44.4 $\pm$ 0.4 & 52.8 \\
MMD & 60.6 $\pm$ 1.1 & 45.9 $\pm$ 0.3 & 57.8 $\pm$ 0.5 & 43.8 $\pm$ 1.2 & 52.0 \\
DANN & 55.2 $\pm$ 1.9 & 47.0 $\pm$ 0.7 & 57.2 $\pm$ 0.9 & 42.9 $\pm$ 0.9 & 50.6 \\
CDANN & 56.3 $\pm$ 2.0 & 47.1 $\pm$ 0.9 & 57.2 $\pm$ 1.1 & 42.4 $\pm$ 0.8 & 50.8 \\
MTL & 58.4 $\pm$ 2.1 & 48.4 $\pm$ 0.8 & 58.9 $\pm$ 0.6 & 43.0 $\pm$ 1.3 & 52.2 \\
SagNet & 56.4 $\pm$ 1.9 & 50.5 $\pm$ 2.3 & 59.1 $\pm$ 0.5 & 44.1 $\pm$ 0.6 & 52.5 \\
ARM & 60.1 $\pm$ 1.5 & 48.3 $\pm$ 1.6 & 55.3 $\pm$ 0.6 & 40.9 $\pm$ 1.1 & 51.2 \\
V-REx & 56.8 $\pm$ 1.7 & 46.5 $\pm$ 0.5 & 58.4 $\pm$ 0.3 & 43.8 $\pm$ 0.3 & 51.4 \\
RSC & 59.9 $\pm$ 1.4 & 46.7 $\pm$ 0.4 & 57.8 $\pm$ 0.5 & 44.3 $\pm$ 0.6 & 52.1 \\
AND-mask & 54.7 $\pm$ 1.8 & 48.4 $\pm$ 0.5 & 55.1 $\pm$ 0.5 & 41.3 $\pm$ 0.6 & 49.8 \\
SAND-mask & 56.2 $\pm$ 1.8 & 46.3 $\pm$ 0.3 & 55.8 $\pm$ 0.4 & 42.6 $\pm$ 1.2 & 50.2 \\
Fishr & 60.4 $\pm$ 0.9 & 50.3 $\pm$ 0.3 & 58.8 $\pm$ 0.5 & 44.9 $\pm$ 0.5 & 53.6 \\
SelfReg & 60.0 $\pm$ 2.3 & 48.8 $\pm$ 1.0 & 58.6 $\pm$ 0.8 & 44.0 $\pm$ 0.6 & 52.8 \\
\midrule
IDM & 60.1 $\pm$ 1.4 & 48.8 $\pm$ 1.9 & 57.9 $\pm$ 0.2 & 44.3 $\pm$ 1.2 & 52.8 \\
\bottomrule
\end{tabular}}
\end{table}

\begin{table}[htbp]
\centering
\small
\caption{Detailed results on DomainNet in DomainBed.}
\label{tbl:domainbed_domainnet}
\adjustbox{max width=\textwidth}{%
\begin{tabular}{lccccccc}
\toprule
\textbf{Algorithm} & \textbf{clip} & \textbf{info} & \textbf{paint} & \textbf{quick} & \textbf{real} & \textbf{sketch} & \textbf{Average} \\
\midrule
ERM & 58.6 $\pm$ 0.3 & 19.2 $\pm$ 0.2 & 47.0 $\pm$ 0.3 & 13.2 $\pm$ 0.2 & 59.9 $\pm$ 0.3 & 49.8 $\pm$ 0.4 & 41.3 \\
IRM & 40.4 $\pm$ 6.6 & 12.1 $\pm$ 2.7 & 31.4 $\pm$ 5.7 & 9.8 $\pm$ 1.2 & 37.7 $\pm$ 9.0 & 36.7 $\pm$ 5.3 & 28.0 \\
GroupDRO & 47.2 $\pm$ 0.5 & 17.5 $\pm$ 0.4 & 34.2 $\pm$ 0.3 & 9.2 $\pm$ 0.4 & 51.9 $\pm$ 0.5 & 40.1 $\pm$ 0.6 & 33.4 \\
Mixup & 55.6 $\pm$ 0.1 & 18.7 $\pm$ 0.4 & 45.1 $\pm$ 0.5 & 12.8 $\pm$ 0.3 & 57.6 $\pm$ 0.5 & 48.2 $\pm$ 0.4 & 39.6 \\
MLDG & 59.3 $\pm$ 0.1 & 19.6 $\pm$ 0.2 & 46.8 $\pm$ 0.2 & 13.4 $\pm$ 0.2 & 60.1 $\pm$ 0.4 & 50.4 $\pm$ 0.3 & 41.6 \\
CORAL & 59.2 $\pm$ 0.1 & 19.9 $\pm$ 0.2 & 47.4 $\pm$ 0.2 & 14.0 $\pm$ 0.4 & 59.8 $\pm$ 0.2 & 50.4 $\pm$ 0.4 & 41.8 \\
MMD & 32.2 $\pm$ 13.3 & 11.2 $\pm$ 4.5 & 26.8 $\pm$ 11.3 & 8.8 $\pm$ 2.2 & 32.7 $\pm$ 13.8 & 29.0 $\pm$ 11.8 & 23.5 \\
DANN & 53.1 $\pm$ 0.2 & 18.3 $\pm$ 0.1 & 44.2 $\pm$ 0.7 & 11.9 $\pm$ 0.1 & 55.5 $\pm$ 0.4 & 46.8 $\pm$ 0.6 & 38.3 \\
CDANN & 54.6 $\pm$ 0.4 & 17.3 $\pm$ 0.1 & 44.2 $\pm$ 0.7 & 12.8 $\pm$ 0.2 & 56.2 $\pm$ 0.4 & 45.9 $\pm$ 0.5 & 38.5 \\
MTL & 58.0 $\pm$ 0.4 & 19.2 $\pm$ 0.2 & 46.2 $\pm$ 0.1 & 12.7 $\pm$ 0.2 & 59.9 $\pm$ 0.1 & 49.0 $\pm$ 0.0 & 40.8 \\
SagNet & 57.7 $\pm$ 0.3 & 19.1 $\pm$ 0.1 & 46.3 $\pm$ 0.5 & 13.5 $\pm$ 0.4 & 58.9 $\pm$ 0.4 & 49.5 $\pm$ 0.2 & 40.8 \\
ARM & 49.6 $\pm$ 0.4 & 16.5 $\pm$ 0.3 & 41.5 $\pm$ 0.8 & 10.8 $\pm$ 0.1 & 53.5 $\pm$ 0.3 & 43.9 $\pm$ 0.4 & 36.0 \\
V-REx & 43.3 $\pm$ 4.5 & 14.1 $\pm$ 1.8 & 32.5 $\pm$ 5.0 & 9.8 $\pm$ 1.1 & 43.5 $\pm$ 5.6 & 37.7 $\pm$ 4.5 & 30.1 \\
RSC & 55.0 $\pm$ 1.2 & 18.3 $\pm$ 0.5 & 44.4 $\pm$ 0.6 & 12.5 $\pm$ 0.1 & 55.7 $\pm$ 0.7 & 47.8 $\pm$ 0.9 & 38.9 \\
AND-mask & 52.3 $\pm$ 0.8 & 17.3 $\pm$ 0.5 & 43.7 $\pm$ 1.1 & 12.3 $\pm$ 0.4 & 55.8 $\pm$ 0.4 & 46.1 $\pm$ 0.8 & 37.9 \\
SAND-mask & 43.8 $\pm$ 1.3 & 15.2 $\pm$ 0.2 & 38.2 $\pm$ 0.6 & 9.0 $\pm$ 0.2 & 47.1 $\pm$ 1.1 & 39.9 $\pm$ 0.6 & 32.2 \\
Fishr & 58.3 $\pm$ 0.5 & 20.2 $\pm$ 0.2 & 47.9 $\pm$ 0.2 & 13.6 $\pm$ 0.3 & 60.5 $\pm$ 0.3 & 50.5 $\pm$ 0.3 & 41.8 \\
SelfReg & 60.7 $\pm$ 0.1 & 21.6 $\pm$ 0.1 & 49.5 $\pm$ 0.1 & 14.2 $\pm$ 0.3 & 60.7 $\pm$ 0.1 & 51.7 $\pm$ 0.1 & 43.1 \\
\midrule
IDM & 58.8 $\pm$ 0.3 & 20.7 $\pm$ 0.2 & 48.3 $\pm$ 0.1 & 13.7 $\pm$ 0.4 & 59.1 $\pm$ 0.1 & 50.2 $\pm$ 0.3 & 41.8 \\
\bottomrule
\end{tabular}}
\end{table}

\bibliography{references}
\bibliographystyle{IEEEtran}

\end{document}